\documentclass{article}
\usepackage{fullpage}
\usepackage{natbib}

\usepackage{abstract}     % Allows abstract customization
\usepackage{wrapfig}

\usepackage{amsmath,amssymb,mathtools}
\usepackage{bm}
\usepackage{amsthm}
\usepackage{graphicx,subfig,color,wrapfig,epsfig,epstopdf}
\usepackage{enumitem}
\usepackage{booktabs}
\usepackage{algorithm,algorithmic}

\usepackage{graphicx}
\usepackage{float}

\newfloat{figtab}{htb}{fgtb}
\makeatletter
  \newcommand\figcaption{\def\@captype{figure}\caption}
  \newcommand\tabcaption{\def\@captype{table}\caption}
\makeatother

\usepackage[utf8]{inputenc} % allow utf-8 input
\usepackage[T1]{fontenc}    % use 8-bit T1 fonts
\usepackage{hyperref}       % hyperlinks
\usepackage{cleveref}
\usepackage{url}            % simple URL typesetting
\usepackage{booktabs}       % professional-quality tables
\usepackage{amsfonts}       % blackboard math symbols
\usepackage{nicefrac}       % compact symbols for 1/2, etc.
\usepackage{microtype}      % microtypography

\def\rc{\color{black}}

\newcommand\numberthis{\addtocounter{equation}{1}\tag{\theequation}}

\allowdisplaybreaks[4]

\newcounter{ass_counter}
\newcounter{thm_counter}

\newtheorem{theorem}[thm_counter]{Theorem}%[section]
\newtheorem{lemma}[thm_counter]{Lemma}%[Lemma]
\newtheorem{corollary}[thm_counter]{Corollary}
\newtheorem{assumption}[ass_counter]{Assumption}

\Crefname{assumption}{Assumption}{Assumptions}

\def\nrc{\color{black}}

\usepackage{hyperref}

\newenvironment{salign*}{\begin{align*}\small}{\end{align*}}

\title{Communication Compression for Decentralized Training}

\usepackage{authblk}
\author[1]{Hanlin Tang\thanks{\texttt{htang14@ur.rochester.edu}}}
\author[2]{Shaoduo Gan\thanks{\texttt{sgan@inf.ethz.ch}}}
\author[2]{Ce Zhang\thanks{\texttt{ce.zhang@inf.ethz.ch}}}

\author[3]{Tong Zhang\thanks{\texttt{tongzhang@tongzhang-ml.org}}}
\author[3,1]{Ji Liu\thanks{\texttt{ji.liu.uwisc@gmail.com}}}
\affil[1]{Department of Computer Science, University of Rochester}
\affil[2]{Department of Computer Science, ETH Zurich}
\affil[3]{Tencent AI Lab}

\begin{document}

\maketitle

\begin{abstract}
%\vspace{-1em}

Optimizing distributed learning systems is an art
of balancing between computation and communication.
There have been two lines of research that try to
deal with slower networks: {\em communication 
compression} for
low bandwidth networks, and {\em decentralization} for
high latency networks. In this paper, We explore
a natural question: {\em can the combination
of both techniques lead to
a system that is robust to both bandwidth
and latency?}

Although the system implication of such combination
is trivial, the underlying theoretical principle and
algorithm design is challenging:  unlike centralized algorithms, simply compressing
{\rc exchanged information,
even in an unbiased stochastic way, 
within the decentralized network would accumulate the error and fail to converge.} 
In this paper, we develop
a framework of compressed, decentralized training and
propose two different strategies, which we call
{\em extrapolation compression} and {\em difference compression}.
We analyze both algorithms and prove 
both converge at the rate of $O(1/\sqrt{nT})$ 
where $n$ is the number of workers and $T$ is the
number of iterations, matching the convergence rate for
full precision, centralized training. We validate 
our algorithms and find that our proposed algorithm outperforms
the best of merely decentralized and merely quantized
algorithm significantly for networks with {\em both} 
high latency and low bandwidth.
\end{abstract}

%%\vspace{-1em}
%\vspace{-1.5em}
\section{Introduction}
%\vspace{-1em}
%%\vspace{-2mm}

%<<<<<<< 
%Decentralized distributed parallel algorithms \citep{zhang2017projection,Lian_dsgd,Sirb_dsgd} are mostly applied to solve the consensus problem \citep{zhao2016decentralized}, where the network topology is decentralized. A recent work shows that decentralized algorithms could outperform the centralized counterpart for distributed training\citep{Lian_dsgd}. The main advantage of decentralized algorithms over centralized algorithms lies on avoiding the communication traffic in the central node. In particular, decentralized algorithms could be much more efficient than centralized algorithms when the network bandwidth is small and the latency is large. Here the question we want to ask: \emph{can we further blackuce the communication cost for decentralized training?}
%=======
%%\vspace{-0.5em}

When training machine learning models in a distributed
fashion, the underlying constraints of how workers (or nodes)
communication have a significant impact on the training
algorithm. When workers cannot form a fully connected
communication topology % (maybe due to physical constraints),
or the communication latency is high ({\rc 
e.g., in sensor networks or mobile networks}), 
decentralizing the
communication comes to the rescue. On the other hand,
when the amount of data sent through the network is an
optimization objective (maybe to lower the cost or energy consumption), 
or the network bandwidth is low,
compressing the traffic,
either via sparsification \citep{wangni2017gradient, konevcny2016randomized} or quantization \citep{zhang2017zipml, pmlr-v70-suresh17a} is a popular strategy. 
In this paper, our goal is to develop a 
novel framework that works robustly in an environment
that {\em both} decentralization and communication
compression could be beneficial. {\rc In this paper, we focus on quantization, the process of lowering the precision of data representation, often in a
stochastically unbiased way. But the same techniques
would apply to other unbiased compression schemes such as
sparsification.%~\cite{wangni2017gradient}.}

Both decentralized training and quantized (or compressed more generally) training have
attracted intensive interests recently \citep{Yuan_dsgd, zhao2016decentralized, Lian_dsgd, konevcny2016randomized, alistarh2017qsgd}. 
Decentralized algorithms usually exchange local models 
among nodes, which consumes the main communication budget; on
the other hand, quantized algorithms usually exchange
quantized gradient, and update an un-quantized model.
A straightforward idea to combine these two is to directly 
quantize the models sent through the network during 
decentralized training. However, this simple strategy does not
converge to the right solution as the quantization error would
accumulate during training. The technical contribution of
this paper is to develop novel algorithms that combine {\em both}
decentralized training and quantized training together.

\paragraph{ Problem Formulation.} We consider the following decentralized optimization:
%\vspace{-0.5em}
{\small
\begin{equation}
\min_{x\in\mathbb{R}^{N}}\quad f(x) = {1\over n} \sum_{i=1}^n \underbrace{\mathbb{E}_{\xi\sim\mathcal{D}_i}F_{i}(x; \xi)}_{=: f_i(x)},\label{eq:main}
\end{equation}
}
where $n$ is the number of node and $\mathcal{D}_i$ is the local data distribution for node $i$. $n$ nodes form a connected graph and each node can only communicate with its neighbors. {\nrc Here we only assume $f_i(x)$'s are with L-Lipschitzian gradients.}

\paragraph{ Summary of Technical Contributions.}
In this paper, we propose two decentralized parallel stochastic gradient descent algorithms (D-PSGD): extrapolation compression D-PSGD (ECD-PSGD) and 
difference compression D-PSGD (DCD-PSGD). Both algorithms can be proven to converge in the rate roughly $O(1/\sqrt{nT})$ where $T$ is the number of iterations. The convergence rates are consistent with two special cases: centralized parallel stochastic gradient descent (C-PSGD) and D-PSGD. {\rc To the best of our knowledge, this is the first work to combine quantization algorithms and decentralized algorithms for generic optimization.}

The key difference between ECD-PSGD and 
DCD-PSGD is that DCD-PSGD quantizes the {\em difference} between
the last two local models, and ECD-PSGD quantizes the 
{\em extrapolation} between the last two local models. DCD-PSGD admits a slightly better convergence rate than ECD-PSGD when the data variation among nodes is very large. On the other hand, ECD-PSGD is more
robust to more aggressive quantization, as extremely low precision 
quantization can cause DCD-PSGD to diverge, since DCD-PSGD has strict constraint on quantization. 
%On the other hand, DCD-PSGD could be useful when the change of the model is  sparse (and thus the difference), while ECD-PSGD could be better when the model is sparse (and thus the extrapolation).
In this paper, we analyze both algorithms, and empirically
validate our theory. We also show that when the underlying
network has both high latency and low bandwidth, 
both algorithms outperform state-of-the-arts 
significantly. {\rc We present both algorithm because
we believe both of them are theoretically interesting. In
practice, ECD-PSGD could potentially be a more robust
choice.}

\paragraph{Definitions and notations}
Throughout this paper, we use following notations and definitions:
\begin{itemize}[fullwidth]
\item $\nabla f(\cdot)$ denotes the gradient of a function $f$.
\item $f^{*}$ denotes the optimal solution of \eqref{eq:main}.
\item $\lambda_{i}(\cdot)$ denotes the $i$-th largest eigenvalue of a matrix.
\item $\bm{1}=[1,1,\cdots,1]^{\top}\in\mathbb{R}^n$ denotes the full-one vector.
\item $\|\cdot\|$ denotes the $l_2$ norm for vector.
\item $\|\cdot\|_F$ denotes the vector Frobenius norm of matrices.
\item $\bm{C}(\cdot)$ denotes the compressing operator.
\item $f_i(x) := \mathbb{E}_{\xi\sim\mathcal{D}_i}F_{i}(x; \xi)$.
\end{itemize}

%\vspace{-0.5em}
\section{Related work}
%\vspace{-1em}

\paragraph{Stochastic gradient descent} The \textsl{Stocahstic Gradient Descent} (\textbf{SGD}) \citep{Ghadimi_dsgd,Moulines_dsgd,Nemi_dsgd} - a stochastic variant of the gradient descent method - has been widely used for solving large scale machine learning problems \citep{Leon_sgd}. It admits the optimal convergence rate $O ( 1/\sqrt{T} )$ {\nrc for non-convex functions}. \\
\paragraph{Centralized algorithms} The centralized algorithms is a widely used scheme for parallel computation, such as Tensorflow \citep{abadi2016tensorflow}, MXNet \citep{chen2015mxnet}, and CNTK \citep{Seide:2016:CMO:2939672.2945397}. It uses a central node to control all leaf nodes. For \textsl{Centralized Parallel Stochastic Gradient Descent} (\textbf{C-PSGD}), the central node performs parameter updates and leaf nodes compute stochastic gradients based on local information in parallel. In \citet{agarwal2011distributed,zinkevich2010parallelized}, the effectiveness of C-PSGD is studied with latency taken into consideration. The distributed mini-batches SGD, which requires each leaf node to compute the stochastic gradient more than once before the parameter update, is studied in \citet{dekel2012optimal}. \citet{recht2011hogwild} proposed a variant of C-PSGD, HOGWILD, and proved that it would still work even if we allow the memory to be shared and let the private mode to be overwriten by others. The asynchronous non-convex C-PSGD optimization is studied in \citet{lian2015asynchronous}. \citet{zheng2016asynchronous} proposed an algorithm to improve the performance of the asynchronous C-PSGD. In \citet{alistarh2017qsgd,de2017understanding}, a quantized SGD is proposed to save the communication cost for both convex and non-convex object functions. The convergence rate for C-PSGD is $O (1/\sqrt{Tn}) $. The tradeoff between the mini-batch number and the local SGD step is studied in \citet{DBLP:journals/corr/abs-1808-07217,Stich18local}.\\
\paragraph{Decentralized algorithms} 
Recently, decentralized training algorithms have attracted significantly
amount of attentions. Decentralized algorithms 
are mostly applied to solve the consensus problem \citep{zhang2017projection,Lian_dsgd,Sirb_dsgd}, where the network topology is decentralized. A recent work shows that decentralized algorithms could outperform the centralized counterpart for distributed training \citep{Lian_dsgd}. The main advantage of decentralized algorithms over centralized algorithms lies on avoiding the communication traffic in the central node. In particular, decentralized algorithms could be much more efficient than centralized algorithms when the network bandwidth is small and the latency is large.
The decentralized algorithm (also named gossip algorithm in some literature under certain scenarios \citep{colin2016gossip}) only assume a connect computational network, without using the central node to collect information from all nodes. Each node owns its local data and can only exchange information with its neighbors. The goal is still to learn a model over all distributed data.
The decentralized structure can applied in solving of multi-task multi-agent reinforcement learning \citep{omidshafiei2017deep,mhamdi2017dynamic}. \citet{Boyd_dsgd} uses a randomized {\nrc weighted} matrix and studied the effectiveness of the {\nrc weighted} matrix in different situations. Two methods \citep{li2017decentralized,Shi_dgd} were proposed to blackuce the steady point error in decentralized gradient descent convex optimization. \citet{dobbe2017fully} applied an information theoretic framework for decentralize analysis. The performance of the decentralized algorithm is dependent on the second largest eigenvalue of the {\nrc weighted} matrix. In \citet{NIPS2018_7705}, In \citet{NIPS2018_7705}, they proposed the gradient descent based algorithm (\textbf{CoLA}) for decentralized learning of linear classification and regression models, and proved the convergence rate for strongly convex and general convex cases.\\
\paragraph{Decentralized parallel stochastic gradient descent} The \textsl{Decentralized Parallel Stochastic Gradient Descent} (\textbf{D-PSGD}) \citep{nedic2009distributed,Yuan_dsgd} requires each node to exchange its own stochastic gradient and update the parameter using the information it receives. In \citet{nedic2009distributed}, the convergence rate for a time-varying topology was proved when the maximum of the subgradient is assumed to be bounded. In \citet{Lan_dsgd}, a new decentralized primal-dual type method is proposed with a computational complexity of {\nrc$O(\sqrt{n/T})$} for general convex objectives. The linear speedup of D-PSGD is proved in \citet{Lian_dsgd}, where the computation complexity is {\nrc$O (1/\sqrt{nT})$}. The asynchronous variant of D-PSGD is studied in \citet{Lian_adsgd}. \\
\paragraph{Compression} To guarantee the convergence and correctness, this paper only considers using the unbiased stochastic compression techniques. Existing methods include randomized quantization \citep{zhang2017zipml, pmlr-v70-suresh17a} and randomized sparsification \citep{wangni2017gradient, konevcny2016randomized}. Other compression methods can be found in~\citet{kashyap2007quantized,lavaei2012quantized,nedic2009quan}. In \citet{DBLP:conf/nips/DrumondLJF18}, a compressed DNN training algorithm is proposed. In \citet{NIPS2018_7697}, a centralized biased sparsified parallel SGD with memory is studied and proved to admits an factor of acceleration.%, {\rc which are only applicable to the concensus problems.}

\begin{wrapfigure}{R}{10cm}
\centering
\vspace{-0.5cm}
%\vspace{-1em}
\includegraphics[width=7.5cm]{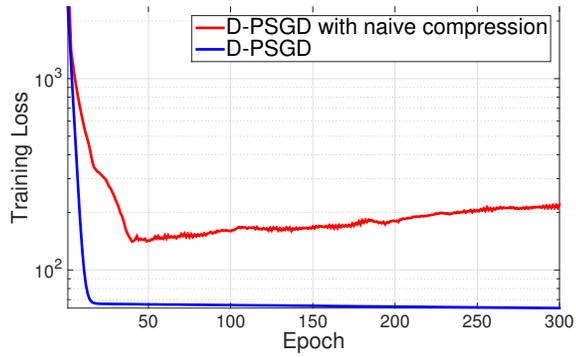}
\vspace{0cm}
\caption{D-PSGD vs. D-PSGD with naive compression }
\vspace{0.5cm}
\label{Fig:naive}
%\vspace{-1em}
\end{wrapfigure}

\section{Preliminary: decentralized parallel stochastic gradient descent (D-PSGD)}
%\vspace{-1em}
Unlike the traditional (centralized) parallel stochastic gradient descent (C-PSGD), which requires a central node to compute the average value of all leaf nodes, the decentralized parallel stochastic gradient descent (D-PSGD) algorithm does not need such a central node. Each node (say node $i$) only exchanges its local model $\bm{x}^{(i)}$ with its neighbors to take weighted average, specifically, $\bm{x}^{(i)} = \sum_{j=1}^nW_{ij}\bm{x}^{(j)}$ where $W_{ij} \geq 0$ in general and $W_{ij}=0$ means that node $i$ and node $j$ is not connected.
%The non-zero element of {\nrc weighted} matrix $W$ means that node $i$ and node $j$ is connected, with $W_{ij}$ represents the weight.\\
At $t$th iteration, D-PSGD consists of three steps ($i$ is the node index):

%\vspace{-0.5em}
{\bf 1.} Each node computes the stochastic gradient $\nabla F_i(\bm{x}^{(i)}_t;\xi_t^{(i)})$, where $\xi^{(i)}_t$ is the samples from its local data set and $\bm{x}^{(i)}_t$ is the local model on node $i$.

%\vspace{-0.5em}
{\bf 2.} Each node queries its neighbors' variables and updates its local model using $\bm{x}^{(i)} = \sum_{j=1}^nW_{ij}\bm{x}^{(j)}$.

%\vspace{-0.5em}
{\bf 3.} Each node updates its local model {\small$\bm{x}^{(i)}_t \gets \bm{x}^{(i)}_t - \gamma_t \nabla F_i\left(\bm{x}_t^{(i)};\xi^{(i)}_t\right) $} using stochastic gradient, where $\gamma_t$ is the learning rate.

%%\vspace{-2mm}
To look at the D-PSGD algorithm from a global view, by defining{\small
\begin{align*}
&X :=  [\bm{x}^{(1)}, \bm{x}^{(2)}, \cdots, \bm{x}^{(n)}] \in \mathbb{R}^{N\times n},\quad
G(X; \xi) :=  [\nabla F_1(x^{(1)}; \xi^{(1)}), \cdots, \nabla F_n(x^{(n)}; \xi^{(n)})] \\
%& \in \mathbb{R}^{N\times n},\\
 &\nabla f(\overline{X}) := \sum_{i=1}^{n}\frac{1}{n}\nabla f_i\left(\frac{1}{n}\sum_{i=1}^nx^{(i)}\right),\quad
\overline{\nabla f}(X):=  \mathbb{E}_{\xi}G(X;\xi_t)\frac{\bm{1}}{n}=\frac{1}{n}\sum_{i=1}^n\nabla f_i(x^{(i)}),
\end{align*}}
the D-PSGD can be summarized into the form
%\begin{align*}
$ X_{t+1} = X_t W - \gamma_t G(X_t; \xi_t)$.
%\end{align*}
%where subscript $t$ denotes the iteration index. 

The convergence rate of D-PSGD can be shown to be {\small$O \left(\frac{\sigma}{\sqrt{nT}} + \frac{n^{\frac{1}{3}} \zeta^{\frac{2}{3}}}{T^{\frac{2}{3}}}\right) $} (without assuming  convexity) where both $\sigma$ and $\zeta$ are the stochastic variance (please refer to Assumption~\ref{ass:global} for detailed definitions), if the learning rate is chosen appropriately.

%\vspace{-0.5em}
\section{Quantized, Decentralized Algorithms}
%\vspace{-1em}
We introduce two quantized decentralized algorithms that compress information exchanged between nodes. All communications for decentralized algorithms are exchanging local models $\bm{x}^{(i)}$. 

To reduce the communication cost, a straightforward idea is to compress the information exchanged within the decentralized network just like centralized algorithms sending compressed stochastic gradient \citep{2016arXiv161002132A}. Unfortunately, such naive combination does not work even using the unbiased stochastic compression and diminishing learning rate as shown in Figure~\ref{Fig:naive}. The reason can be seen from the detailed derivation (please find it in Supplement).

%\paragraph{Naive combination with compression does not work} Consider combine compression with the D-PSGD algorithm. Let the compression of exchanged models $X_t$ be
%\[
%\tilde{X}_t = C(X_t) = X_t + Q_t,
%\]
%where $Q_t=[\bm{q}_t^{(1)},\bm{q}_t^{(2)},\cdots,\bm{q}_t^{(n)}]$, and $\bm{q}_t^{(i)}=\tilde{\bm{x}}_t^{(i)}-\bm{x}^{(i)}_t$ is the random noise. Then the update iteration becomes
%\begin{align*}
%X_{t+1} = &\tilde{X}_t W - \gamma_t G(X_t; \xi_t)\\
% = & X_t W + \underbrace{Q_tW}_{\text{not diminish}}- \gamma_t G(X_t; \xi_t).
%\end{align*}
%This naive combination does not work, because the compression error $Q_t$ does not diminish unlike the stochastic gradient variance that can be controlled by $\gamma_t$ either decays to zero or is chosen to be small enough. 
 Before propose our solutions to this issue, let us first make some common optimization assumptions for analyzing decentralized stochastic algorithms~\citep{Lian_adsgd}.
\begin{assumption}
\label{ass:global}
Throughout this paper, we make the following commonly used assumptions:
%\vspace{-1mm}
\begin{enumerate}

  \item \textbf{Lipschitzian gradient:} All function $f_i(\cdot)$'s are with $L$-Lipschitzian gradients.
  \item \textbf{Symmetric double stochastic matrix:} The  weighted matrix $W$ is a real double stochastic matrix that satisfies $W=W^{\top}$ and $W\bm{1}=W$.
  \item \textbf{Spectral gap:} Given the symmetric doubly stochastic matrix $W$,
    we define $\rho := \max \{| \lambda_2 (W) |, | \lambda_n (W) |\}$ and assume $\rho<1$.
  \item \textbf{Bounded variance:} Assume the variance of stochastic gradient to be bounded%\vspace{-2mm} %for any $x$ in each node $i$:
{\small\begin{align*}
    \mathbb{E}_{\xi\sim \mathcal{D}_i} \left\| \nabla F_i (\bm{x}; \xi) - \nabla f_i (\bm{x})\right\|^2 \leqslant  \sigma^2, %\quad \forall i, \forall \bm{x},
    \quad
%<<<<<<< HEAD
     {1\over n}\sum_{i=1}^n\left\| \nabla f_i (\bm{x})-\nabla f (\bm{x})\right\|^2 \leqslant  \zeta^2, \quad \forall i, \forall \bm{x},
%=======
%     {{1\over n}\sum_{i=1}^n}\left\| \nabla f_i (\bm{x})-\nabla f (\bm{x})\right\|^2 \leqslant & \zeta^2, \quad \forall i, \forall \bm{x},
%>>>>>>> 23f3417dec316d4293a38ecbb01e9ce680df4d7b
\end{align*}}
  %\vspace{-3mm}\item \textbf{Start from 0:} We assume $X_1 = 0$. This assumption simplifies the proof w.l.o.g.
  \item \textbf{Independent and unbiased stochastic compression:} The stochastic compression operation $\bm{C}(\cdot)$ is unbiased, that is, $\mathbb{E}(\bm{C}(Z)) = Z$ for any $Z$ and the stochastic compressions are independent on different workers or at different time point.
  \end{enumerate}
\end{assumption}
%\vspace{-2mm}
The last assumption essentially restricts the compression to be lossy but unbiased. Biased stochastic compression is generally hard to ensure the convergence and lossless compression can combine with any algorithms. Both of them are beyond of the scope of this paper. The commonly used stochastic unbiased compression include random quantization\footnote{A real number is randomly quantized into one of closest thresholds, for example, givens the thresholds $\{0, 0.3, 0.8, 1\}$, the number ``$0.5$'' will be quantized to $0.3$ with probability $40\%$ and to $0.8$ with probability $60\%$. Here, we assume that all numbers have been normalized into the range $[0,1]$.} \citep{zhang2017zipml} and sparsification\footnote{A real number $z$ is set to $0$ with probability $1-p$ and to $z/p$ with probability $p$.} \citep{wangni2017gradient,konevcny2016randomized}.

%{\rc
%\paragraph{Using $\bm{z}$ value for information exchanging} Since sending the compressed local model fails to work, so we send an alternative value, so called $\bm{z}$ value, for communication. In the following algorithms, at each iteration $t$, each worker would need to compute $\bm{z}_t^{(i)}$ (the $\bm{z}$ value of worker $i$) independently and send it to neighbors. For simplicity, we define
%\begin{align*}
%Z :=  [\bm{z}_1, \bm{z}_2, \cdots, \bm{z}_n] \in \mathbb{R}^{N\times n},\quad
%C(Z) :=  [\bm{C}(\bm{z}_1), \bm{C}(\bm{z}_2), \cdots, \bm{C}(\bm{z}_n)] \in \mathbb{R}^{N\times n}
%\end{align*}
%for future discussion.
%
%Before introducing the specific algorithms, we make the following  assumptions that could fit in many existing compressing methods \citep{nedic2009quan, konevcny2016randomized, wangni2017gradient}.
%
%
%\begin{assumption}We make the following assumption for the compress operator throughout this paper:
%\begin{enumerate} 
% \item \textbf{Unbiasedness of the compressing operator:} For each worker, we assume that
% $\mathbb{E}_{x}(\bm{C}(\bm{x})) = x$,
%which means the expectation of the noise is zero on condition of the original value $\bm{x}$.
% \item \textbf{Independence of the compressing procedure:} We assume that each worker $i$ compress the original value $\bm{x}^{(i)}$ independently with each other.
%\end{enumerate}
%\end{assumption}
%}

%\vspace{-0.5em}
\subsection{Difference compression approach}
%\vspace{-1em}

\begin{figtab}[t!]
\vspace{-0.4cm}
\scriptsize

\begin{minipage}[b]{0.43\textwidth}
\begin{algorithm}[H]\label{alg_2}
\caption{DCD-PSGD}\label{alg2}
%\begin{minipage}{1.0\linewidth}
\begin{algorithmic}[1]
\scriptsize
\STATE {\bfseries Input:} Initial point $\bm{x}^{(i)}_1=\bm{x}_1$, initial replica $\hat{\bm{x}}^{(i)}_1=\bm{x}_1$, iteration step length $\gamma$, {\nrc weighted} matrix $W$, and number of total iterations T
\FOR{t = 1,2,...,T}
\STATE Randomly sample $\xi^{(i)}_t$ from local data of the $i$th node
\STATE Compute local stochastic gradient $\nabla F_i(\bm{x}^{(i)}_t;\xi^{(i)}_t)$ using $\xi^{(i)}_t$ and current optimization variable $\bm{x}^{(i)}_t$
\STATE \label{alg:step} Update the local model using local stochastic gradient and the weighted average of its connected neighbors' replica {\rc (denote as $\hat{\bm{x}}^{(j)}_t$)}:%\vspace{-3mm}
\begin{align*}
\bm{x}_{t+\frac{1}{2}}^{(i)}=\sum_{j=1}^{n}W_{ij}\bm{x}^{(j)}_t -\gamma\nabla F_i(\bm{x}^{(i)}_t;\xi^{(i)}_t),
\end{align*}%\vspace{-3mm}
\STATE Each node computes
$\bm{z}^{(i)}_{t} = \bm{x}_{t+{1\over 2}}^{(i)} - \bm{x}^{(i)}_t,$
and compress this $\bm{z}^{(i)}_{t}$ into $\bm{C}(\bm{z}_t^{(i)})$.
\STATE Update the local optimization variables %\vspace{-1.5mm}
\begin{align*}
\bm{x}_{t+1}^{(i)}\gets \bm{x}_{t}^{(i)} + \bm{C}(\bm{z}_t^{(i)}).
\end{align*}%\vspace{-1.5mm}
\STATE Send $\bm{C}(\bm{z}_t^{(i)})$ to its connected neighbors, and update the replicas of its connected neighbors' values:%\vspace{-2mm}
\begin{align*}
\hat{\bm{x}}_{t+1}^{(j)} = \hat{\bm{x}}_{t}^{(j)} + \bm{C}(\bm{z}_t^{(i)}).
\end{align*}%\vspace{-4mm}
\ENDFOR
\STATE {\bfseries Output:} $\frac{1}{n}\sum_{i=1}^{n}\bm{x}^{(i)}_T$
\end{algorithmic}
%\end{minipage}
\end{algorithm}
\end{minipage}\quad
\begin{minipage}[b]{0.5\textwidth}

%%\vspace{-1.3em}
\begin{algorithm}[H]
{
\caption{ECD-PSGD}\label{alg1}
\begin{algorithmic}[1]
\scriptsize
\STATE {\bfseries Input:} Initial point $\bm{x}^{(i)}_1=\bm{x}_1$, initial estimate ${\rc \tilde{\bm{x}}}^{(i)}_1=\bm{x}_1$, iteration step length $\gamma$, {\nrc weighted} matrix $W$, and number of total iterations T.
\FOR{$t = 1,2,\cdots,T$}
\STATE Randomly sample $\xi^{(i)}_t$ from local data of the $i$th node
\STATE Compute local stochastic gradient $\nabla F_i(\bm{x}^{(i)}_t;\xi^{(i)}_t)$ using $\xi^{(i)}_t$ and current optimization variable $\bm{x}^{(i)}_t$
\STATE Compute the neighborhood weighted average by using the estimate value of the connected neighbors :%\vspace{-2mm}
\[ %\begin{equation}
\bm{x}_{t+\frac{1}{2}}^{(i)}=\sum_{j=1}^{n}W_{ij}{\rc \tilde{\bm{x}}}^{(j)}_t
\]%\vspace{-2mm} %\end{equation}
\STATE Update the local model%\vspace{-2mm}
\[
\bm{x}_{t+1}^{(i)}\gets \bm{x}_{t+\frac{1}{2}}^{(i)}-\gamma\nabla F_i(\bm{x}^{(i)}_t,\xi^{(i)}_t)
\]%\vspace{-3mm}
\STATE Each node computes the $z$-value of itself:%\vspace{-2mm}
\[
\bm{z}^{(i)}_{t+1} = \left(1-0.5t\right)\bm{x}_t^{(i)}+0.5t\bm{x}_{t+1}^{(i)}%\vspace{-3mm}
\]
and compress this $\bm{z}^{(i)}_{t}$ into $\bm{C}(\bm{z}_t^{(j)})$.
\STATE Each node updates the estimate for its connected neighbors:
\[{\rc \tilde{\bm{x}}}_{t+1}^{(j)}=\left(1-2t^{-1}\right){\rc \tilde{\bm{x}}}^{(j)}_t+2t^{-1}\bm{C}(\bm{z}_t^{(j)})
\]
\ENDFOR
\STATE {\bfseries Output:} $\frac{1}{n}\sum_{i=1}^{n}\bm{x}^{(i)}_T$
\end{algorithmic}
}
\end{algorithm}
\end{minipage}
%%\vspace{-0.6cm}
\end{figtab}

 In this section, we introduces a difference based approach, namely, difference compression D-PSGD (DCD-PSGD), to ensure efficient convergence. 

The DCD-PSGD basically follows the framework of D-PSGD, except that nodes exchange the compressed difference of local models between two successive iterations, instead of exchanging local models. More specifically, each node needs to store its neighbors' models in last iteration $\{{\rc{\hat{\bm{x}}}^{(j)}_{t}}: j~\text{is node $i$'s neighbor}\}$ and follow the following steps:
% Unlike Algorithm~\ref{alg1}, where each node can only hold an estimate of its neighbors' local models, in this algorithm we enable each node to have the exact value of its neighbors' local model by keep tracking of its neighbors' value change. The reduction of the communication cost is realized by compressing the value change. So the whole procedure can be expressed as
\begin{enumerate} 
\item take the weighted average and apply stochastic gradient descent step: 
$\bm{x}_{t+\frac{1}{2}}^{(i)}=\sum_{j=1}^{n}W_{ij}{{\rc\hat{\bm{x}}}^{(j)}_t} -\gamma\nabla F_i(\bm{x}^{(i)}_t;\xi^{(i)}_t)$,
where {\rc$\hat{\bm{x}}^{(j)}_t$} is just the replica of $\bm{x}^{(j)}_t$ but is stored on node $i$\footnote{Actually each neighbor of node $j$ maintains a replica of $\bm{x}^{(j)}_t$.};
\item compress the difference between $\bm{x}^{(i)}_{t}$ and $\bm{x}_{t+\frac{1}{2}}^{(i)}$ and update the local model:$\bm{z}_t^{(i)} = \bm{x}_{t+\frac{1}{2}}^{(i)} -\bm{x}_t^{(i)},\quad \bm{x}^{(i)}_{t+1} =\bm{x}^{(i)}_t + \bm{C}(\bm{z}_t^{(i)})$;
\item send $\bm{C}(\bm{z}_t^{(i)})$ and query neighbors' $\bm{C}(\bm{z}_t)$ to update the local replica:
$\forall j~\text{is node $i$'s neighbor}$,
$\hat{\bm{x}}_{t+1}^{(j)} = \hat{\bm{x}}^{(j)}_t + \bm{C}(\bm{z}_t^{(j)})$.
\end{enumerate}

%\begin{align*}
%\bm{x}_{t+1}^{(i)} = & \bm{x}_{t}^{(i)} + \bm{Q}\left(\bm{x}_{t+ \frac{1}{2}}^{(i)} - \bm{x}_t^{(i)}\right),
%\end{align*}
%where $\bm{Q}\left(\bm{x}_{t+ \frac{1}{2}}^{(i)} - \bm{x}_t^{(i)}\right)$ is the compressed value of $\left(\bm{x}_{t+ \frac{1}{2}}^{(i)} - \bm{x}_t^{(i)}\right)$. 
The full DCD-PSGD algorithm is described in Algorithm~\ref{alg2}.

To {\rc ensure convergence,} we need to make some restriction on the compression operator $\bm{C}(\cdot)$. Again this compression operator could be random quantization or random sparsification or any other operators. We  introduce the definition of the signal-to-noise related parameter $\alpha$. 
Let $\alpha := \sqrt{\sup_{Z\neq 0} {\|Q\|^2_F / \|Z\|^2_F}}$, where $Q=Z-C(Z)$. We have the following theorem.

\begin{theorem}\label{theo_2}
Under the Assumption~\ref{ass:global}, % and ~\ref{ass:alg2}, 
if $\alpha$ satisfies $(1-\rho)^2-4\mu^2\alpha^2>0$, choosing 
%$\gamma_t$ in Algorithm~\ref{alg2} to be a constant 
$\gamma$ {satisfying} $1-3D_1L^2\gamma^2>0$, we have the following convergence rate for
  Algorithm~\ref{alg2}{\small
\begin{align*}
& \sum_{t=1}^T\left(\left(1-D_3\right)\mathbb{E}\|\nabla f(\overline{X}_t)\|^2 + D_4\mathbb{E}\|\overline{\nabla f}(X_t)\|^2\right) 
\leq  \frac{2(f(0)-f^*)}{\gamma} + \frac{L\gamma T\sigma^2}{n}\\
& + \left(\frac{T\gamma^2LD_2}{2} + \frac{\left(4L^2 + 3L^3D_2\gamma^2\right)D_1T\gamma^2}{1-3D_1L^2\gamma^2}\right)\sigma^2 
 + \frac{\left(4L^2 + 3L^3D_2\gamma^2\right)3D_1T\gamma^2}{1-3D_1L^2\gamma^2}\zeta^2, \numberthis \label{bound_theo_2}
\end{align*}}
where $\mu :=  \max_{i\in\{2,\cdots ,n\}}|\lambda_i-1|$, and {\small
\begin{alignat*}{2}
&D_1 :=  \frac{2\alpha^2}{1-\rho^2}\left(\frac{2\mu^2(1+2\alpha^2)}{(1-\rho)^2 - 4\mu^2\alpha^2}+1\right) + \frac{1}{(1-\rho)^2},
&\quad &
D_2 :=  2\alpha^2\left(\frac{2\mu^2(1+2\alpha^2)}{(1-\rho)^2 - 4\mu^2\alpha^2}+1\right)\\
&D_3 :=  \frac{\left(4L^2 + 3L^3D_2\gamma^2\right)3D_1\gamma^2}{1-3D_1L^2\gamma^2} + \frac{3LD_2\gamma^2}{2},
&\quad &
D_4 :=  \left(1-L\gamma\right).
\end{alignat*}}
\end{theorem}
%\vspace{-2mm}
To make the result more clear, we appropriately choose the steplength in the following:

\begin{corollary} \label{cor:convergence_alg2}
Choose {\small $\gamma = \left(6\sqrt{D_1}L + 6\sqrt{D_2L}+\frac{\sigma}{\sqrt{n}}T^{\frac{1}{2}} + \zeta^{\frac{2}{3}}T^{\frac{1}{3}}\right)^{-1}$}
in Algorithm~\ref{alg2}. If $\alpha$ is small enough that satisfies $(1-\rho)^2-4\mu^2\alpha^2>0$, then we have{\small
\begin{align*}
\frac{1}{T}\sum_{t=1}^T\mathbb{E}\|\nabla f(\overline{X}_t)\|^2 \lesssim & \frac{\sigma}{\sqrt{nT}} + \frac{\zeta^{\frac{2}{3}}}{T^{\frac{2}{3}}}+ \frac{1}{T}.
\end{align*}
}
where $D_1$, $D_2$ follow to same definition in Theorem~\ref{theo_2} and we treat $f(0)- f^*$, $L$, and $\rho$ constants.
\end{corollary}
The leading term of the convergence rate is {\small$O \left(1/\sqrt{Tn}\right)$}, and we also proved the convergence rate for {\small$\mathbb{E}\left[\sum_{i=1}^n\left\|\overline{X}_t-\bm{x}_t^{(i)}\right\|^2\right]$} (see \eqref{finalcoro_1} in Supplementary). We shall see the tightness of our result in the following discussion.

\paragraph{Linear speedup} Since the leading term of the convergence rate is {\small$O \left(1/\sqrt{Tn}\right)$} when $T$ is large, which is consistent with the convergence rate of C-PSGD, this indicates that we would achieve a linear speed up with respect to the number of nodes.

\paragraph{Consistence with D-PSGD} Setting $\alpha = 0$ to match the scenario of D-PSGD, ECD-PSGD admits the rate {\small$O \left(\frac{\sigma}{\sqrt{nT}} + \frac{ \zeta^{\frac{2}{3}}}{T^{\frac{2}{3}}}\right)$}, that is slightly better the rate of D-PSGD proved in \citet{Lian_adsgd} {\small$O \left(\frac{\sigma}{\sqrt{nT}} + \frac{n^{\frac{2}{3}} \zeta^{\frac{2}{3}}}{T^{\frac{2}{3}}}\right) $}. {\rc The non-leading terms' dependence of the spectral gap $(1-\rho)$ is also consistent with the result in D-PSDG.}

%\vspace{-0.5em}
\subsection{Extrapolation compression approach}
%\vspace{-1em}

{\rc From Theorem~\ref{theo_2}, we can see that there is an upper bound for the compressing level $\alpha$ in DCD-PSGD. Moreover, since the spectral gap $(1-\rho)$ would decrease with the growth of the amount of the workers, so DCD-PSGD will fail to work under a very aggressive compression. So in this section, we propose another approach, namely ECD-PSGD, to remove the restriction of the compressing degree, with a little sacrifice on the computation efficiency.}

For ECD-PSGD, we make the following assumption that the noise brought by compression is bounded.
\begin{assumption}
  \label{ass:alg1} 
    (\textbf{Bounded compression noise}) We assume the noise due to compression is unbiased and its variance is bounded, that is, $\forall \bm{z}\in \mathbb{R}^n$%\vspace{-1.5mm}
      \begin{align*}
\mathbb{E}\|\bm{C}(\bm{z}) - \bm{z}\|^2\leq \tilde{\sigma}^2/2, \quad  \forall \bm{z}
  \end{align*}
  %where $\bm{q}_t^{(i)} := \bm{z}_t^{(i)} - \bm{C}(\bm{z}_t^{(i)})$
%  \begin{align*}
%\mathbb{E}\|{\rc \bm{q}}_t^{(i)}\|^2\leq \tilde{\sigma}^2/2, \quad \forall i, \forall \bm{z}, \forall t,
%  \end{align*}
%  where $\bm{q}_t^{(i)} := \bm{z}_t^{(i)} - \bm{C}(\bm{z}_t^{(i)})$
\end{assumption}
%\vspace{-3.5mm}
Instead of sending the local model $\bm{x}^{(i)}_t$ directly to neighbors, we send a $z$-value that is extrapolated from $\bm{x}^{(i)}_t$ and $\bm{x}^{(i)}_{t-1}$ at each iteration. Each node (say, node $i$) estimates its neighbor's values $\bm{x}^{(j)}_t$ from compressed $z$-value at $t$-th iteration. This procedure could ensure diminishing estimate error, in particular, {\small$\mathbb{E}\|\tilde{\bm{x}}^{(j)}_t  - \bm{x}^{(j)}_t\|^2 \leq \mathcal{O}\left(t^{-1}\right)$.}

At $t$th iteration, node $i$ performs the following steps to estimate $\bm{x}^{(j)}_t$ by $\tilde{\bm{x}}^{(j)}_t$:

\begin{itemize}%[fullwidth]
\setlength{\itemsep}{1.5pt}
\setlength{\parsep}{2pt}
\setlength{\parskip}{1pt}
\item The node $j$, computes the $z$-value that is obtained through extrapolation{\small
\begin{align*}
\bm{z}_{t}^{(j)}=\left(1-0.5t\right)\bm{x}_{t-1}^{(j)}+0.5t{\bm{x}}_{t}^{(j)},\numberthis \label{alg1_noisecontrol1}
\end{align*}}
\item Compress $\bm{z}_t^{(j)}$ and send it to its neighbors, {\rc say node $i$. Node $i$ computes $\tilde{\bm{x}}_{t}^{(j)}$ using}{\small
\begin{align*}
\tilde{\bm{x}}_{t}^{(j)} = \left(1-2t^{-1}\right)\tilde{\bm{x}}^{(j)}_{t-1}+2t^{-1}\bm{C}(\bm{z}_t^{(j)}). \numberthis \label{alg1_noisecontrol2}
\end{align*}}
\end{itemize}
Using Lemma~\ref{lemma3} (see Supplemental Materials), if the compression noise ${\rc \bm{q}}^{(j)}_t:= \bm{z}_t^{(i)} - \bm{C}(\bm{z}_t^{(i)})$ {\nrc is globally bounded variance} by $\tilde{\sigma}^2/2$, we will have{\small
\begin{align*}
\mathbb{E}(\|\tilde{\bm{x}}^{(j)}_t - \bm{x}^{(j)}_t\|^2) \leq \tilde{\sigma}^2/t.
\end{align*}}
Using this way to estimate the neighbors' local models leads to the following equivalent updating form%\vspace{-2mm}{\small
\begin{align*}
X_{t+1} = & \tilde{X}_tW - \gamma_t G(X_t; \xi_t)
=  X_t  W + \underbrace{{Q}_tW}_{\text{diminished estimate error}} - \gamma_t G(X_t; \xi_t).% \numberthis % \label{global:evaluation}
\end{align*}}
The full extrapolation compression D-PSGD (ECD-PSGD) algorithm is summarized in Algorithm~\ref{alg1}.

Below we will show that EDC-PSGD algorithm would admit the same convergence rate and the same computation complexity as D-PSGD.

\begin{theorem}[Convergence of Algorithm \ref{alg1}] \label{theo_1}
  \label{thm:conv-bounded-variance}
  Under Assumptions~\ref{ass:global} and ~\ref{ass:alg1}, choosing $\gamma_t$ in Algorithm~\ref{alg1} to be constant $\gamma$ {satisfying} {\small$1-6C_1L^2\gamma^2>0$},  we have the following convergence rate for
  Algorithm~\ref{alg1} {\small
\begin{align*}
& \sum_{t=1}^T\left(\left(1-C_3\right)\mathbb{E}\|\nabla f(\overline{X}_t)\|^2 + C_4\mathbb{E}\|\overline{\nabla f}(X_t)\|^2\right)
\\ \leq & \frac{2(f(0)-f^*)}{\gamma} +
\frac{L\log T}{n\gamma}\tilde{\sigma}^2 + \frac{LT\gamma}{n}\sigma^2 + \frac{4C_2\tilde{\sigma}^2L^2}{1-\rho^2}\log T + 4L^2C_2\left(\sigma^2+3\zeta^2\right)C_1T\gamma^2. \numberthis \label{eq_theo_1}
\end{align*}}
where {\small
$C_1 :=  \frac{1}{(1-\rho)^2}$, 
$C_2 :=  \frac{1}{1-6\rho^{-2}C_1L^2\gamma^2}$,
$C_3 :=  12L^2C_2C_1\gamma^2$, and
$C_4 :=  1-L\gamma$. }
%{\small
%\begin{alignat*}{2}
%C_1 :=  \frac{1}{(1-\rho)^2},\quad 
%C_2 :=  \frac{1}{1-6\rho^{-2}C_1L^2\gamma^2},\quad
%C_3 :=  12L^2C_2C_1\gamma^2,\quad 
%C_4 :=  1-L\gamma.
%\end{alignat*}}
\end{theorem}
%%\vspace{-3mm}
To make the result more clear, we choose the steplength in the following:
\begin{corollary} \label{cor:convergence}
In Algorithm~\ref{alg1} choose the steplength {\small$\gamma=\left(12\sqrt{C_1}L+\frac{\sigma}{\sqrt{n}}T^{\frac{1}{2}} + \zeta^{\frac{2}{3}}T^{\frac{1}{3}}\right)^{-1}$}.  Then it admits the following convergence rate (with $f(0)- f^*$, $L$, and $\rho$ treated as constants). {\small
\begin{align*}
\frac{1}{T}\sum_{t=1}^T\mathbb{E}\|\nabla f(\overline{X}_t)\|^2
\lesssim & \frac{\sigma(1+\frac{\tilde{\sigma}^2\log T}{n})}{\sqrt{nT}} + \frac{\zeta^{\frac{2}{3}}(1+\frac{\tilde{\sigma}^2\log T}{n})}{T^{\frac{2}{3}}} + \frac{1}{T}  + \frac{\tilde{\sigma}^2\log{T}}{T}, \numberthis
\label{eq:cor:alg1}
\end{align*}}
\end{corollary}
This result suggests the algorithm converges roughly in the rate {\small$O({1 / \sqrt{nT}})$}, and we also proved the convergence rate for {\small$\mathbb{E}\left[\sum_{i=1}^n\left\|\overline{X}_t-\bm{x}_t^{(i)}\right\|^2\right]$} (see \eqref{final_coro2} in Supplementary). The followed analysis will bring more detailed interpretation to show the tightness of our result.

\paragraph{Linear speedup} Since the leading term of the convergence rate is {\small$O (1/\sqrt{nT})$} when $T$ is large, which is consistent with the convergence rate of C-PSGD, this indicates that we would achieve a linear speed up with respect to the number of nodes.
%\item [Consistency with gradient descent] If all randomness has been removed, that is, $\sigma=0$, $\tilde{\sigma}=0$, and $\zeta=0$, the convergence achieves $O(1/T)$, that is consistent with the rate of gradient descent. 

\paragraph{Consistence with D-PSGD} Setting $\tilde{\sigma} = 0$ to match the scenario of D-PSGD, ECD-PSGD admits the rate {\small$O \left(\frac{\sigma}{\sqrt{nT}} + \frac{\zeta^{\frac{2}{3}}}{T^{\frac{2}{3}}}\right)$}, that is slightly better the rate of D-PSGD proved in \citet{Lian_adsgd} {\small$O \left(\frac{\sigma}{\sqrt{nT}} + \frac{n^{\frac{1}{3}} \zeta^{\frac{2}{3}}}{T^{\frac{2}{3}}}\right) $}. {\rc The non-leading terms' dependence of the spectral gap $(1-\rho)$ is also consistent with the result in D-PSDG.}

\paragraph{Comparison between DCD-PSGD and ECD-PSGD}
On one side, in term of the convergence rate, ECD-PSGD is slightly worse than DCD-PSGD due to additional terms {\small$\left( \frac{\sigma\tilde{\sigma}^2\log{T}}{n\sqrt{nT}} + \frac{\zeta^{\frac{2}{3}}\tilde{\sigma}^2\log{T}}{nT^{\frac{2}{3}}} + \frac{\tilde{\sigma}^2\log{T}}{T}\right)$} that suggests that if $\tilde{\sigma}$ is relatively large than $\sigma$, the additional terms dominate the convergence rate.
%could dominate the convergence rate when $\tilde{\sigma}$ or $\zeta$ is extremely large. 
On the other side, DCD-PSGD does not allow too aggressive compression or quantization and may lead to diverge due to {\small$\alpha \leq \frac{1-\rho}{2\sqrt{2}\mu}$}, while ECD-PSGD is quite robust to aggressive compression or quantization.

%{\rc
%\subsection{Comparison between DCD-PSGD and ECD-PSGD}
%On one side, in term of the convergence rate, ECD-PSGD is slightly worse than DCD-PSGD due to additional terms in $O \left( \frac{\sigma\tilde{\sigma}^2\log{T}}{n\sqrt{nT}} + \frac{\zeta^{\frac{2}{3}}\tilde{\sigma}^2\log{T}}{nT^{\frac{2}{3}}} + \frac{\tilde{\sigma}^2\log{T}}{T}\right)$ that suggests that if $\tilde{\sigma}$ is relatively large than $\sigma$, the additional terms dominate the convergence rate.
%%could dominate the convergence rate when $\tilde{\sigma}$ or $\zeta$ is extremely large. 
%On the other side, DCD-PSGD does not allow too aggressive compression or quantization and may lead to diverge due to $\alpha \leq \frac{1-\rho}{2\sqrt{2}\mu}$, while ECD-PSGD is quite robust to aggressive compression or quantization.
%}
 
\begin{figure*}[t]
\centering
%\vspace{-10mm}
\includegraphics[scale=0.14]{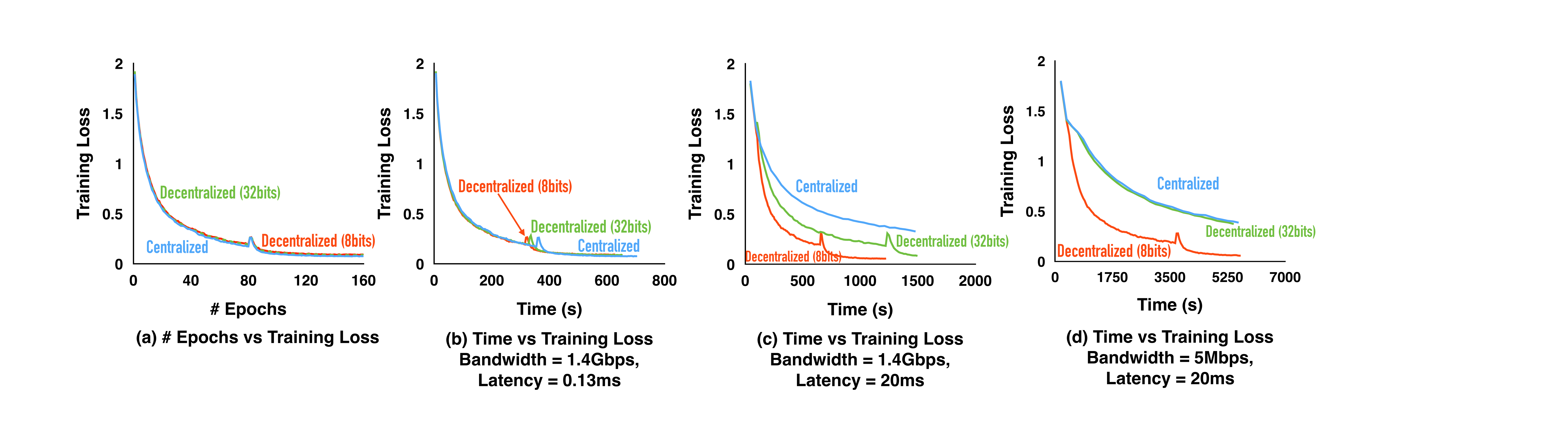}
%\vspace{-0.5em}
\caption{Performance Comparison between Decentralized and AllReduce implementations.}
\label{Fig:TrainingLoss}
%\vspace{-1em}
\end{figure*}

%\vspace{-0.5em}
\section{Experiments}
%\vspace{-1em}

In this section we evaluate two decentralized algorithms by comparing 
with an Allreduce implementation of centralized SGD. We run experiments 
under diverse network conditions and show that, decentralized algorithms with low precision can speed up training without hurting 
convergence.

%\vspace{-0.5em}
\subsection{Experimental Setup}
%\vspace{-1em}

We choose the image classification task as a benchmark to evaluate our theory. We train ResNet-20~\citep{he2016deep} on CIFAR-10 dataset which has 50,000 images for training and 10,000 images for testing. Two proposed algorithms are implemented in Microsoft CNTK and compablack with CNTK's original implementation of distributed SGD:

\begin{figure*}[t]
\centering
%\vspace{-5mm}
\includegraphics[scale=0.14]{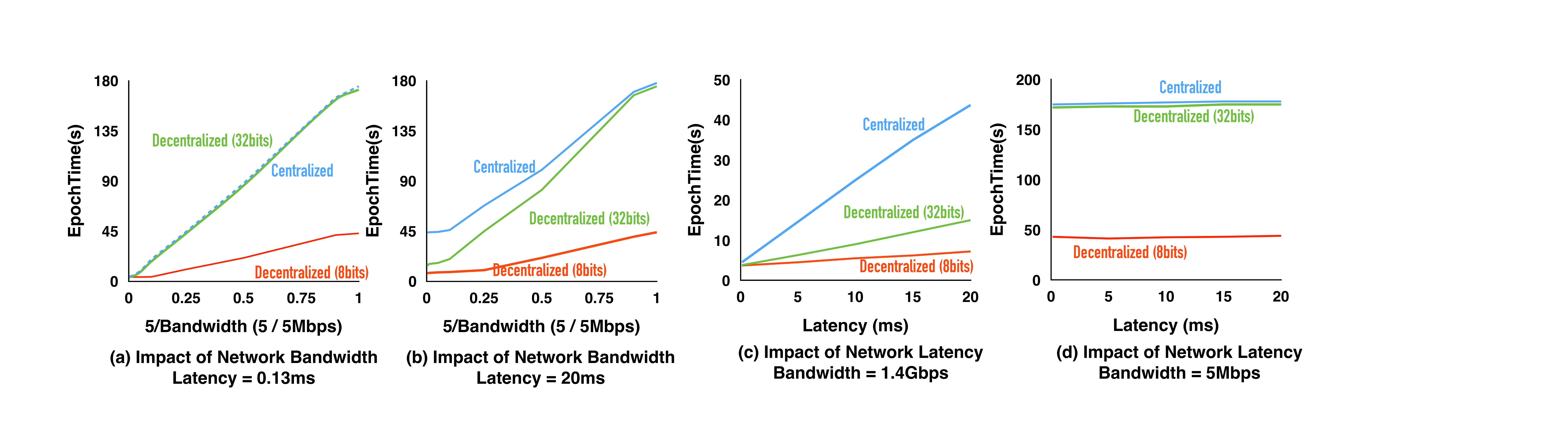}
%\vspace{-0.6em}
\caption{Performance Comparison in Diverse Network Conditions.}
\label{Fig:EpochTime}
%\vspace{-1.5em}
\end{figure*}

%\vspace{-0.5em}
\begin{itemize}[fullwidth]
%\vspace{-0.5em}
\item{\bf{Centralized:}} This implementation is based on MPI Allreduce primitive with full precision (32 bits). It is the standard training method for multiple nodes in CNTK.
%\vspace{-0.5em}
\item{\bf{Decentralized\_32bits/8bits:}} The implementation of the proposed decentralized approach with OpenMPI. The full precision is 32 bits, and the compressed precision is 8 bits.  
%\vspace{-0.5em}
\item In this paper, we omit the comparison with quantized centralized
training because the difference between Decentralized 8bits
and Centralized 8bits would be similar to the original 
decentralized training paper~\cite{Lian_dsgd} -- when the network
latency is high, decentralized algorithm outperforms centralized
algorithm in terms of the time for each epoch.  
\end{itemize}

%\vspace{-0.5em}
We run experiments on 8 Amazon $p2.xlarge$ EC2 instances, each of which has one Nvidia K80 GPU. We use each GPU as a node. In decentralized cases, 8 nodes are connected as a ring topology, which means each node just communicates with its two neighbors. The batch size for each node is same as the default configuration in CNTK. We also tune learning rate for each variant.

%\vspace{-0.5em}
\subsection{Convergence and Run Time Performance}
%\vspace{-1em}

We first study the convergence of our algorithms. Figure~\ref{Fig:TrainingLoss}(a) shows the convergence w.r.t \# epochs of centralized and decentralized cases. 
We only show ECD-PSGD in the figure (and call it Decentralized) because
DCD-PSGD has almost identical convergence behavior in this experiment.
We can see that with our algorithms, decentralization and compression 
would not hurt the convergence rate.

We then compare the runtime performance. Figure \ref{Fig:TrainingLoss}(b, c, d) demonstrates how training loss decreases with the run time under different network conditions. We use $tc$ command to change bandwidth and latency of the underlying network. By default, 1.4 Gbps bandwidth and 0.13 ms latency is the best network condition we can get in this cluster. On this occasion, all implementations have a very similar runtime performance because communication is not the bottleneck for system. When the latency is high, as shown in \ref{Fig:TrainingLoss}(c), decentralized algorithms in both low and full precision can outperform the Allreduce method because of fewer number of communications. However, in low bandwidth case, training time is mainly dominated by the amount of communication data, so low precision method can be obviously faster than these full precision methods.

%\vspace{-0.5em}
\subsection{Speedup in Diverse Network Conditions}
%\vspace{-1em}

To better understand the influence of bandwidth and latency on speedup, we compare the time of one epoch under various of network conditions. Figure \ref{Fig:EpochTime}(a, b) shows the trend of epoch time with bandwidth decreasing from 1.4 Gbps to 5 Mbps. When the latency is low (Figure \ref{Fig:EpochTime}(a)), low precision algorithm is faster than its full precision counterpart because it only needs to exchange around one fourth of full precision method's data amount. Note that although in a decentralized way, full precision case has no advantage over Allreduce in this situation, because they exchange exactly the same amount of data. When it comes to high latency shown in Figure \ref{Fig:EpochTime}(b), both full and low precision cases are much better than Allreduce in the beginning. But also, full precision method gets worse dramatically with the decline of bandwidth.

Figure \ref{Fig:EpochTime}(c, d) shows how latency influences the epoch time under good and bad bandwidth conditions. When bandwidth is not the bottleneck (Figure \ref{Fig:EpochTime}(c)), decentralized approaches with both full and low precision have similar epoch time because they have same number of communications. As is expected, Allreduce is slower in this case. When bandwidth is very low (Figure \ref{Fig:EpochTime}(d)), only decentralized algorithm with low precision can achieve best performance among all implementations.

\begin{figure}
\centering
%\vspace{-2.3em}
\includegraphics[scale=0.15]{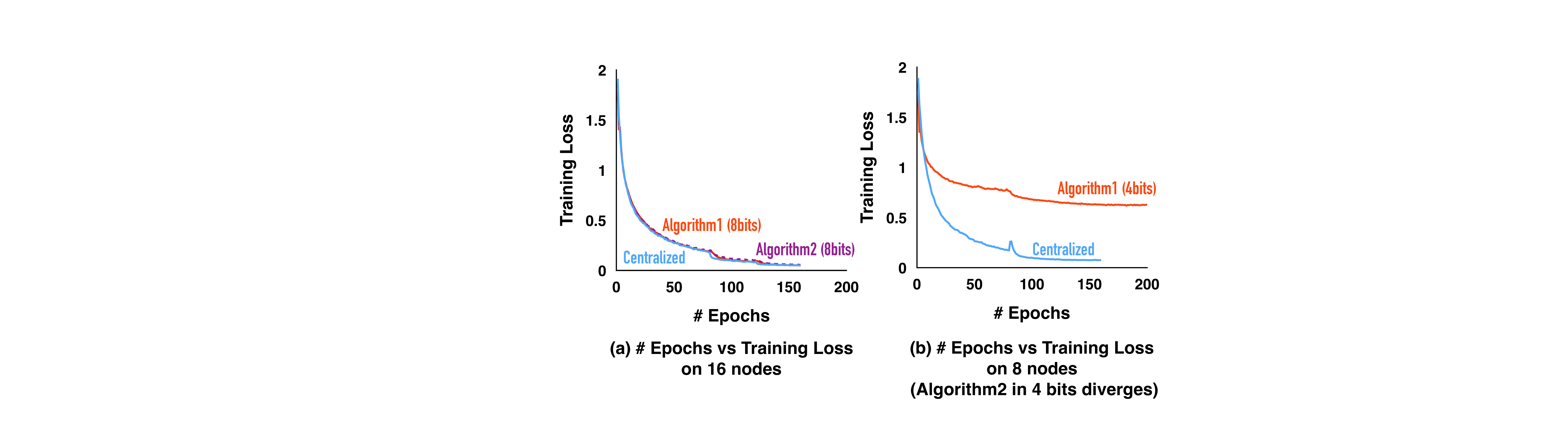}
%\vspace{-2em}
\caption{Comparison of Alg. 1 and Alg. 2 }
%\vspace{-2mm}
\label{Fig:4bits}
\end{figure}

%\vspace{-0.5em}
\subsection{Discussion}
%\vspace{-1em}

Our previous experiments validate the efficiency of the decentralized algorithms on 8 nodes with 8 bits. However, we wonder if we can scale it to more nodes or compress the exchanged data even more aggressively. We firstly conducted experiments on 16 nodes with 8 bits as before. According to Figure \ref{Fig:4bits}(a), Alg. 1 and Alg. 2 on 16 nodes can still achieve basically same convergence rate as Allreduce, which shows the scalability of our algorithms. However, they can not be comparable to Allreduce with 4 bits, as is shown in \ref{Fig:4bits}(b). What is noteworthy is that these two compression approaches have quite different behaviors in 4 bits. For Alg. 1, although it converges much slower than Allreduce, its training loss keeps reducing. However, Alg. 2 just diverges in the beginning of training. This observation is consistent with our theoretical analysis.

%\vspace{-1em}
\section{Conclusion}
%\vspace{-1em}

In this paper, we studied the problem of combining two tricks
of training distributed stochastic gradient descent under
imperfect network conditions: quantization and decentralization.
We developed two novel algorithms or quantized, decentralized
training, analyze the theoretical property of both algorithms,
and empirically study their performance in a various settings
of network conditions. We found that when the underlying communication
networks has {\em both} high latency and low bandwidth, 
quantized, decentralized algorithm outperforms other strategies
significantly.

\newpage
\bibliographystyle{abbrvnat}
\bibliography{zipmlbib}

% \newpage
% \bibliographystyle{abbrvnat}
% \bibliography{references.bib}

\newpage
\onecolumn
\appendix

\begin{center}
{\Huge \bf
Supplemental Materials: Proofs
}
\end{center}

For the convenience of analysis, let us reformulate the updating rule of both proposed algorithms. Both algorithms can be generalized into the following form
\begin{align*}
X_{t+1} =  X_tW-\gamma_tG(X_t;\xi_t) +Q_t, \numberthis \label{general_eq}
\end{align*}
where $Q_t$ is the noise caused by compression.

More specifically, for ECD-PSGD, from \eqref{alg1_noisecontrol1} and \eqref{alg1_noisecontrol2},
we have
\begin{align*}
Q_t = & \tilde{X}_t - X_t,
\end{align*}
where $\tilde{X}_t $ is the estimation of $X_t$ computed using extrapolation. For DCD-PSGD, from \eqref{alg2_globalevolution}, we have
\begin{align*}
\Delta X_t := &X_tW-\gamma F(X_t;\xi_t) -X_t\\
Q_t = & \bm{C}\left(\Delta X_t \right) - \Delta X_t,
\end{align*}
where $\bm{C}(\cdot)$ is the compression operator.

\paragraph{Proof organization} In section~\ref{general}, we first provide the properties of updating rule \eqref{general_eq}. Sections~\ref{sec2} and~\ref{sec3} will next specify $Q_t$ and show the upper bounds for $\|Q_t\|^2$ in Algorithms~\ref{alg1} and~\ref{alg2} respectively.

\paragraph{Notations}
We define some additional notations throughout the following proof
\begin{align*}
\overline{X}_t:= & X_t\frac{\bm{1}}{n}=\frac{1}{n}\sum_{i=1}^n\bm{x}_t^{(i)},\\
\overline{Q}_t:= & Q_t\frac{\bm{1}}{n}=\frac{1}{n}\sum_{i=1}^n\bm{q}_t^{(i)},\\
\overline{G}(X_t;\xi_t):= & G(X_t,\xi_t)\frac{\bm{1}}{n}=\frac{1}{n}\sum_{i=1}^n\nabla F_i(\bm{x}_t^{(i)};\xi_t^{(i)}).
\end{align*}

\section{General bound with compression noise} \label{general}
In this section, to see the influence of the compression more clearly, we are going to prove two general bounds (see see Lemma~\ref{lemma_bound_all_X_ave} and Lemma~\ref{lemma:boundfplus}) for compressed D-PSGD that has the same updating rule like \eqref{general_eq}. Those bounds are very helpful for the following proof of our algorithms. 

The most challenging part of a decentralized algorithm, unlike the centralized algorithm, is that we need to ensure the local model on each node to converge to {\rc the average value $\overline{X}_t$}. So we start with an analysis of the quantity $\left\|\overline{X}_t-\bm{x}_t^{(i)}\right\|^2$ and its influence on the final convergence rate. For both ECD-PSGD and DCD-PSGD, we are going to prove that
\begin{align*}
\sum_{t=1}^{T}\sum_{i=1}^{n}\mathbb{E}\left\|\overline{X}_t-\bm{x}_t^{(i)}\right\|^2
\leq & \frac{2}{1-\rho^2}\sum_{t=1}^{T}\|Q_t\|^2_F + \frac{2}{(1-\rho)^2}\sum_{t=1}^{T}\gamma_t^2\|G(X_{t};\xi_{t})\|^2_F,
\end{align*}
and
\begin{align*}
\mathbb{E}\|\nabla f(\overline{X}_t)\|^2  +  (1-L\gamma_t)\mathbb{E}\|\overline{\nabla f}(X_t)\|^2  \leq & \frac{2}{\gamma_t}\left(\mathbb{E}f(\overline{X}_t)  -  \mathbb{E}f(\overline{X}_{t+1})\right) +  
	\frac{L^2}{n}\mathbb{E}\sum_{i=1}^n\left\|\overline{X}_t-\bm{x}_t^{(i)}\right\|^2 \\
		& + \frac{L}{\gamma_t}\mathbb{E}\|\overline{Q}_t\|^2  +   \frac{L\gamma_t}{n}\sigma^2.
\end{align*}

From the above two inequalities, we can see that the extra noise term decodes the convergence efficiency of $\bm{x}_t^{(i)}$ to the average $\overline{X}_t$.

The proof of the general bound for \eqref{general_eq} is divided into two parts. In subsection~\ref{secA1}, we provide a new perspective in understanding decentralization, which can be very helpful for the simplicity of our following proof. In subsection~\ref{secA2}, we give the detail proof for the general bound.

\subsection{A more intuitive way to understand decentralization}\label{secA1}
To have a better understanding of how decentralized algorithms work, and how can we ensure a consensus from all local variables on each node. We provide a new perspective to understand decentralization using coordinate transformation, which can simplify our analysis in the following.

The confusion matrix $W$ satisfies $W = \sum_{i=1}^{n}\lambda_i\bm{v}^i\left(\bm{v}^i\right)^{\top}$ is doubly stochastic, so we can decompose it into $W = P\Lambda P^{\top}$, where $P = \left(\bm{v}^1,\bm{v}^2,\cdots,\bm{v}^n\right)$ that satisfies $P^{\top}P=PP^{\top}=I$. Without the loss of generality, we can assume $\lambda_1\geq \lambda_2 \geq \cdots, \geq \lambda_{n}$. Then we have the following equalities:
\begin{align*}
X_{t+1} = & X_tW-\gamma_tG(X_t;\xi_t) +Q_t ,\\
X_{t+1} = & X_tP\Lambda P^{\top}-\gamma_tG(X_t;\xi_t) +Q_t ,\\ 
X_{t+1}P = & X_tP\Lambda - G(X_t;\xi_t)P + Q_tP.
\end{align*}
Consider the coordinate transformation using $P$ as the base change matrix, and
denote $Y_t = X_tP$, $H(X_t;\xi_t) = G(X_t;\xi_t)P$, $R_t = Q_tP$. Then the above equation can be rewritten as
\begin{align*}
Y_{t+1} = & Y_t\Lambda -H(X_t;\xi_t) + R_t,\numberthis\label{Y_t}.
\end{align*}
Since $\Lambda$ is a diagonal matrix, so we use $\bm{y}_t^{(i)}$, $\bm{h}_t^{(i)}$, $\bm{r}_t^{(i)}$ to indicate the $i$-th column of $Y_t$, $H(X_t;\xi_t)$, $R_t$. Then \eqref{Y_t} becomes
\begin{align*}
\bm{y}_{t+1}^{(i)} = & \lambda_i\bm{y}_{t}^{(i)} - \bm{h}_t^{(i)} + \bm{r}_t^{(i)}, \quad \forall i \in \{1,\cdots,n\} \numberthis\label{y_t}.
\end{align*}
\eqref{y_t} offers us a much intuitive way to analysis the algorithm. Since all eigenvalues of $W$, except $\lambda_1$, satisfies $|\lambda_i|<1$, so the corresponding $\bm{y}_t^{(i)}$ would ``decay to zero'' due to the scaling factor $\lambda_i$. 

Moreover, since the eigenvector corresponding to $\lambda_1$ is $\frac{1}{\sqrt{n}}(1,1,\cdots,1)$, then we have $\bm{y}_t^{(1)} = \overline{X}_t\sqrt{n}$. So, if $t \to \infty$, intuitively we can set $\bm{y}_t^{(i)}\to \bm{0}$ for $i \neq 1$, then $Y_t\to (\overline{X}_t\sqrt{n},0,\cdots,0)$ and $X_t\to \overline{X}_t\frac{\bm{1}}{n} $. This whole process shows how the confusion works under a coordinate transformation. 
%We dollow the definition for $Y$, $H$, $R$ and $\Lambda$ in the proof below.

\subsection{Analysis for the general updating form in \eqref{general_eq} }\label{secA2}
\begin{lemma}\label{lemma_bound_substract_mean}
For any matrix $X_t\in \mathbb{R}^{N\times n}$, decompose the confusion matrix $W$ as $W = \sum_{i=1}^n \lambda_i\bm{v}^{(i)}\left(\bm{v}^{(i)}\right)^{\top} = P\Lambda P^{\top}$, where $P = (\bm{v}^{(1)},\bm{v}^{(2)},\cdots,\bm{v}^{(n)})\in \mathbb{R}^{N\times n}$, $\bm{v}^{(i)}$ is the normalized eigenvector of $\lambda_i$, and $\Lambda$ is a diagonal matrix with $\lambda_i$ be its $i$th element. We have
\begin{align*}
\sum_{i=1}^n\left\| X_tW^t\bm{e}^{(i)}-X_t\frac{\bm{1}_n}{n} \right\|^2 =& \left\| X_tW^t-X_t\bm{v}^{(1)}\left(\bm{v}^{(1)}\right)^{\top}\right\|^2_F \leq
\left\| \rho^{t}X_t\right\|^2_F,
\end{align*}
where $\rho$ follows the defination in Theorem~\ref{theo_1}.
\end{lemma}
\begin{proof}
Since $W^t = P\Lambda^t P^{\top}$, we have
\begin{align*}
\sum_{i=1}^n\left\| X_tW^t\bm{e}^{(i)}-X_t\frac{\bm{1}_n}{n} \right\|^2 = & \sum_{i=1}^n\left\| \left(X_tW^t-X_t\frac{\bm{1}_n\bm{1}_n^{\top}}{n}\right)\bm{e}^{(i)} \right\|^2\\
= & \left\| X_tW^t-X_t\bm{v}^{(1)}\left(\bm{v}^{(1)}\right)^{\top}\right\|^2_F\\
= & \left\| X_tP\Lambda^tP^{\top}-X_tP\begin{pmatrix}
1,0,\cdots,0\\0,0,\cdots,0\\ \cdots \\0,0,\cdots,0
\end{pmatrix}P^{\top}\right\|^2_F\\
= & \left\| X_tP\Lambda^t-X_tP\begin{pmatrix}
1,0,\cdots,0\\0,0,\cdots,0\\ \cdots \\0,0,\cdots,0
\end{pmatrix}\right\|^2_F\\
= & \left\| X_tP\begin{pmatrix}
&0,&0,&0,&\cdots,&0\\
&0,&\lambda_2^t,&0,&\cdots,&0\\
&0,&0,&\lambda_3^t,&\cdots,&0\\
& \hdotsfor{5}\\
&0,&0,&0,&\cdots,&\lambda_n^t
\end{pmatrix}\right\|^2_F\\
\leq & \left\|\rho^{t}X_tP\right\|^2_F\\
= & \left\|\rho^{t}X_t\right\|^2_F.
\end{align*}
Specifically, when $t=0$, we have
\begin{align*}
\sum_{i=1}^n\left\| X_t\bm{e}^{(i)}-X_t\frac{\bm{1}_n}{n} \right\|^2 = & \left\| X_tP\begin{pmatrix}
&0,&0,&0,&\cdots,&0\\
&0,&1,&0,&\cdots,&0\\
&0,&0,&1,&\cdots,&0\\
& \hdotsfor{5}\\
&0,&0,&0,&\cdots,&1
\end{pmatrix}\right\|^2_F\\
= & \sum_{i=2}^n \left\|\bm{y}_t^{(i)}\right\|^2,\numberthis \label{lemma_bound_substract_mean_eq0}
\end{align*}
where $\bm{y}_t^{(i)} = X_tP\bm{e^{(i)}}$.
\end{proof}

\begin{lemma}\label{lemma1}
Given two non-negative sequences $\{a_t\}_{t=1}^{\infty}$ and $\{b_t\}_{t=1}^{\infty}$ that satisfying
\begin{equation}
a_t =  \sum_{s=1}^t\rho^{t-s}b_{s}, \numberthis \label{eqn1}
\end{equation}
with $\rho\in[0,1)$, we have
\begin{align*}
S_k:=\sum_{t=1}^{k}a_t   \leq & \sum_{s=1}^k\frac{b_s}{1-\rho},\\
D_k:=\sum_{t=1}^{k}a_t^2 \leq & % \frac{(1+\frac{2\rho}{1-\rho})\sum_{t=1}^kb_t^2-\rho^2a_k^2}{1-\rho^2} \\ \leq & 
\frac{1}{(1-\rho)^2} \sum_{s=1}^kb_s^2.
\end{align*}
%where $S_k=\sum_{t=1}^{k}a_t$, and $D_k=\sum_{t=1}^{k}a_t^2$.
\end{lemma}

\begin{proof}
From the definition, we have
\begin{align*}
S_k= & \sum_{t=1}^{k}\sum_{s=1}^t\rho^{t-s}b_{s}
=  \sum_{s=1}^{k}\sum_{t=s}^k\rho^{t-s}b_{s}
=  \sum_{s=1}^{k}\sum_{t=0}^{k-s}\rho^{t}b_{s}
\leq  \sum_{s=1}^{k}{b_{s}\over 1-\rho}, \numberthis \label{eqn3}\\
D_k=  & \sum_{t=1}^{k}\sum_{s=1}^t\rho^{t-s}b_{s}\sum_{r=1}^t\rho^{t-r}b_{r}\\
= & \sum_{t=1}^{k}\sum_{s=1}^t\sum_{r=1}^t\rho^{2t-s-r}b_{s}b_{r} \\
\leq &  \sum_{t=1}^{k}\sum_{s=1}^t\sum_{r=1}^t\rho^{2t-s-r}{b_{s}^2+b_{r}^2\over2}\\
= & \sum_{t=1}^{k}\sum_{s=1}^t\sum_{r=1}^t\rho^{2t-s-r}b_{s}^2 \\
\leq  & {1\over 1-\rho}\sum_{t=1}^{k}\sum_{s=1}^t\rho^{t-s}b_{s}^2\\
\leq & {1\over (1-\rho)^2}\sum_{s=1}^{k}b_{s}^2. \quad \text{(due to \eqref{eqn3})}
\end{align*}
\end{proof}

Lemma~\ref{lemma_bound_substract_mean} shows us an overall understanding about how the confusion matrix works, while Lemma~\ref{lemma1} is a very important tool for analyzing the sequence in Lemma~\ref{lemma_bound_substract_mean}. Next we are going to give a upper bound for the difference between the local modes and the global mean mode.

\begin{lemma}\label{lemma_bound_all_X_ave}
Under Assumption~\ref{ass:global}, we have
\begin{align*}
\sum_{t=1}^{T}\sum_{i=1}^{n}\mathbb{E}\left\|\overline{X}_t-\bm{x}_t^{(i)}\right\|^2
\leq & \frac{2}{1-\rho^2}\sum_{t=1}^{T}\|Q_s\|^2_F + \frac{2}{(1-\rho)^2}\sum_{t=1}^{T}\gamma_t^2\|G(X_{t};\xi_{t})\|^2_F.
\end{align*}
\end{lemma}
\begin{proof}
From the updating rule, we have
\begin{align*}
X_{t} = &\sum_{s=1}^{t-1}\gamma_s G\left(X_s;\xi_s\right)W^{t-s-1} +\sum_{s=1}^{t-1}Q_sW^{t-s},\\
\overline{X}_t = &\sum_{s=1}^{t-1}\gamma_sG\left(X_s;\xi_s\right)W^{t-s-1}\frac{\bm{1}}{n} +\sum_{s=1}^{t-1}Q_sW^{t-s}\frac{\bm{1}}{n}\\
=& \sum_{s=1}^{t-1}\gamma_s \overline{G}\left(X_s;\xi_s\right) +\sum_{s=1}^{t-1}\overline{Q}_s.\quad \text{(due to $W\frac{\bm{1}}{n} = \frac{\bm{1}}{n}$)}
\end{align*}
Therefore it yields
\begin{align*}
&\sum_{i=1}^n\mathbb{E}\left\|\overline{X}_t-\bm{x}_t^{(i)}\right\|^2\\
 = & \sum_{i=1}^n\mathbb{E}\left\|\sum_{s=1}^{t-1}\left(Q_{s}W^{t-s}\bm{e}^{(i)}-\overline{Q}_s\right)  -  \sum_{s=1}^{t-1}\gamma_s\left( G(X_{s};\xi_{s})W^{t-s-1}\bm{e}^{(i)} - \overline{G}(X_{s};\xi_{s})\right)\right\|^2 \\
\leq &  2\sum_{i=1}^n\mathbb{E}\left\|\sum_{s=1}^{t-1}\left(Q_{s}W^{t-s}\bm{e}^{(i)}-\overline{Q}_s\right)\right\|^2 + 2\sum_{i=1}^n\mathbb{E}\left\|\sum_{s=1}^{t-1}\gamma_s\left( G(X_{s};\xi_{s})W^{t-s-1}\bm{e}^{(i)} - \overline{G}(X_{s};\xi_{s})\right)\right\|^2\\
= & 2\sum_{i=1}^n\mathbb{E}\left\|\sum_{s=1}^{t-1}\left(Q_{s}W^{t-s}\bm{e}^{(i)}-\overline{Q}_s\right)\right\|^2 + 2\sum_{i=1}^n\mathbb{E}\left\|\sum_{s=1}^{t-1}\gamma_s\left( G(X_{s};\xi_{s})W^{t-s-1}\bm{e}^{(i)} - \overline{G}(X_{s};\xi_{s})\right)\right\|^2\\
= & 2\sum_{i=1}^n\sum_{s=1}^{t-1}\mathbb{E}\left\|\left(Q_{s}W^{t-s}\bm{e}^{(i)}-\overline{Q}_s\right)\right\|^2 + 2\sum_{i=1}^n\mathbb{E}\left\|\sum_{s=1}^{t-1}\gamma_s\left( G(X_{s};\xi_{s})W^{t-s-1}\bm{e}^{(i)} - \overline{G}(X_{s};\xi_{s})\right)\right\|^2\\
& + 4\sum_{i=1}^n\sum_{s\neq s'}\mathbb{E}\left\langle \mathbb{E}_{_{Q_s}}Q_{s}W^{t-s}\bm{e}^{(i)}- \mathbb{E}_{_{Q_s}}\overline{Q}_s , \mathbb{E}_{_{Q_{s'}}}Q_{s'}W^{t-s'}e^{(i)}- \mathbb{E}_{_{Q_{s'}}}\overline{Q}_{s'}\right\rangle\\
= & 2\sum_{i=1}^n\sum_{s=1}^{t-1}\mathbb{E}\left\|\left(Q_{s}W^{t-s}\bm{e}^{(i)}-\overline{Q}_s\right)\right\|^2 + 2\sum_{i=1}^n\mathbb{E}\left\|\sum_{s=1}^{t-1}\gamma_s\left( G(X_{s};\xi_{s})W^{t-s-1}\bm{e}^{(i)} - \overline{G}(X_{s};\xi_{s})\right)\right\|^2\\
= & 2\sum_{s=1}^{t-1}\mathbb{E}\left\|\left(Q_{s}W^{t-s}-Q_s\bm{v}_1\bm{v}_1^{\top}\right)\right\|^2_F + 2\mathbb{E}\left\|\sum_{s=1}^{t-1}\gamma_s\left( G(X_{s};\xi_{s})W^{t-s-1} - G(X_{s};\xi_{s})\bm{v}_1\bm{v}_1^{\top}\right)\right\|^2_F\\
\leq & 2\mathbb{E}\sum_{s=1}^{t-1}\left\|\rho^{t-s}Q_s\right\|^2_F  +  2\mathbb{E}\left(
	\sum_{s=1}^{t-1}\gamma_s\rho^{t-s-1}\left\|G(X_{s};\xi_{s})\right\|_F\right)^2, \quad \text{(due to Lemma~\ref{lemma_bound_substract_mean})}
\end{align*}
We can see that $\sum_{s=1}^{t-1}\rho^{2(t-s)}\left\|Q_s\right\|^2_F$ and $\sum_{s=1}^{t-1}\gamma_s\rho^{t-s-1}\left\|G(X_{s};\xi_{s})\right\|_F$ has the same structure with the sequence in Lemma~\ref{lemma1}, which leads to
\begin{align*}
\sum_{t=1}^{T}\sum_{i=1}^{n}\mathbb{E}\left\|\overline{X}_t-\bm{x}_t^{(i)}\right\|^2
\leq & \frac{2}{1-\rho^2}\mathbb{E}\sum_{t=1}^{T}\|Q_s\|^2_F + \frac{2}{(1-\rho)^2}\mathbb{E}\sum_{t=1}^{T}\gamma_t^2\|G(X_{t};\xi_{t})\|^2_F.
\end{align*}
\end{proof}

\begin{lemma}\label{lemma:boundfplus}
Following the Assumption~\ref{ass:global}, we have
\begin{align*}
\frac{\gamma_t}{2}\mathbb{E}\|\nabla f(\overline{X}_t)\|^2  +  (\frac{\gamma_t}{2}-\frac{L\gamma_t^2}{2})\mathbb{E}\|\overline{\nabla f}(X_t)\|^2  \leq & \mathbb{E}f(\overline{X}_t)  -  \mathbb{E}f(\overline{X}_{t+1})  +  
	\frac{L^2\gamma_t}{2n}\mathbb{E}\sum_{i=1}^n\left\|\overline{X}_t-\bm{x}_t^{(i)}\right\|^2 \\
		& + \frac{L}{2}\mathbb{E}\|\overline{Q}_t\|^2  +   \frac{L\gamma_t^2}{2n}\sigma^2.\numberthis \label{lemma:boundfplus_eq}
\end{align*}
\begin{proof}
From the updating rule, we have
\begin{align*}
X_{t+1} = \tilde{X}_tW - \gamma_t G(X_t;\xi_t) = X_tW + Q_tW -\gamma_t G(X_t;\xi_t),
\end{align*}
which implies
\begin{align*}
\overline{X}_{t+1} = & \left(X_tW + Q_tW -\gamma_t G(X_t;\xi_t)\right)\frac{\bm{1}}{n}\\
& = \frac{X_t\bm{1}}{n} + \frac{Q_t\bm{1}}{n} -\gamma_t\frac{G(X_t;\xi_t)\bm{1}}{n}\\
& = \overline{X}_t +\overline{Q}_t - \gamma_t \overline{G}(X_t;\xi_t).
\end{align*}

From the Lipschitzian condition for the objective function $f_i$, we know that $f$ also satisfies the Lipschitzian condition. Then we have
\begin{align*}
&\mathbb{E}f(\overline{X}_{t+1})\\
 \leq & \mathbb{E}f(\overline{X}_t)+\mathbb{E}\left\langle\nabla f(\overline{X_t}), -\gamma_t\overline{G}(X_t;\xi_t)+\overline{Q}_t\right\rangle  +  \frac{L}{2}\mathbb{E}\left\|-\gamma_t\overline{G}(X_t;\xi_t)  +  \overline{Q}_t\right\|^2 \\
 = & \mathbb{E}f(\overline{X}_t)  +  \mathbb{E}\langle\nabla f(\overline{X}_t), -\gamma_t\overline{G}(X_t;\xi_t)+\mathbb{E}_{_{Q_t}}\overline{Q}_t\rangle  \\
& + \frac{L}{2}(\mathbb{E}\|\gamma_t\overline{G}(X_t;\xi_t)\|^2  +  \mathbb{E}\|\overline{Q}_t\|^2   +  \mathbb{E}\langle-\gamma_t\overline{G}(X_t;\xi_t),\mathbb{E}_{_{Q_t}}\overline{Q}_t\rangle)\\
= & \mathbb{E}f(\overline{X}_t)  +  \mathbb{E}\langle\nabla f(\overline{X}_t), -\gamma_t\mathbb{E}_{\xi_t}\overline{G}(X_t;\xi_t)\rangle  +  \frac{L\gamma_t^2}{2}\mathbb{E}\|\overline{G}(X_t;\xi_t)\|^2  +  \frac{L}{2}\mathbb{E}\|\overline{Q}_t\|^2\quad \text{(due to $\mathbb{E}_{Q_t}\overline{Q}_t=\bm{0}$)}\\
= &  \mathbb{E}f(\overline{X}_t)  -  \gamma_t\mathbb{E}\langle\nabla f(\overline{X}_t), \overline{\nabla f}	(X_t)\rangle  +   \frac{L\gamma_t^2}{2}\mathbb{E}\|(\overline{G}(X_t;\xi_t) - \overline{\nabla f}	(X_t))+\overline{\nabla f}(X_t)\|^2  +  \frac{L}{2}\mathbb{E}\|\overline{Q}_t\|^2\\
= & \mathbb{E}f(\overline{X}_t)  -  \gamma_t\mathbb{E}\langle\nabla f(\overline{X}_t), \overline{\nabla f}	(X_t)\rangle  +   \frac{L\gamma_t^2}{2}\mathbb{E}\|\overline{G}(X_t;\xi_t) - \overline{\nabla f}	(X_t)\|^2  +  \frac{L\gamma_t^2}{2}\mathbb{E}\|\overline{\nabla f}(X_t)\|^2 \\ 
& +  L\gamma_t^2\mathbb{E}\langle\mathbb{E}_{\xi_t}\overline{G}(X_t;\xi_t) - \overline{\nabla f}(X_t),\overline{\nabla f}(X_t)\rangle  +  \frac{L}{2}\mathbb{E}\|\overline{Q}_t\|^2\\
= & \mathbb{E}f(\overline{X}_t)  -  \gamma_t\mathbb{E}\langle\nabla f(\overline{X}_t), \overline{\nabla f}	(X_t)\rangle  +   \frac{L\gamma_t^2}{2}\mathbb{E}\|\overline{G}(X_t;\xi_t) - \overline{\nabla f}	(X_t)\|^2\\
&  +  \frac{L\gamma_t^2}{2}\mathbb{E}\|\overline{\nabla f}(X_t)\|^2  +  
	\frac{L}{2}\mathbb{E}\|\overline{Q}_t\|^2\\
    = & \mathbb{E}f(\overline{X}_t)  -  \gamma_t\mathbb{E}\langle\nabla f(\overline{X}_t), \overline{\nabla f}	(X_t)\rangle  +   \frac{L\gamma_t^2}{2n^2}\mathbb{E}\left\|\sum_{i=1}^n\left(\nabla F_i(x_t^{(i)};\xi_t^{(i)}) - \nabla f_i	(x_t^{(i)})\right)\right\|^2\\
    & +  \frac{L\gamma_t^2}{2}\mathbb{E}\|\overline{\nabla f}(X_t)\|^2  +  
	\frac{L}{2}\mathbb{E}\|\overline{Q}_t\|^2\\
    = & \mathbb{E}f(\overline{X}_t)  -  \gamma_t\mathbb{E}\langle\nabla f(\overline{X}_t), \overline{\nabla f}	(X_t)\rangle  +   \frac{L\gamma_t^2}{2n^2}\sum_{i=1}^n\mathbb{E}\left\|\nabla F_i(x_t^{(i)};\xi_t^{(i)}) - \nabla f_i	(x_t^{(i)})\right\|^2\\
    & + \sum_{i\neq i'}^n\mathbb{E}\left\langle \mathbb{E}_{\xi_t}\nabla F_i(x_t^{(i)};\xi_t^{(i)}) - \nabla f_i(x_t^{(i)}),\nabla \mathbb{E}_{\xi_t}F_{i'}(x_t^{(i')};\xi_t^{(i')}) - \nabla f_{i'}	(x_t^{(i')}) \right\rangle\\
    & +  \frac{L\gamma_t^2}{2}\mathbb{E}\|\overline{\nabla f}(X_t)\|^2  +  
	\frac{L}{2}\mathbb{E}\|\overline{Q}_t\|^2\\    
\leq & \mathbb{E}f(\overline{X}_t)  -  \gamma_t\mathbb{E}\langle\nabla f(\overline{X}_t), 
	\overline{\nabla f}(X_t)\rangle  +    \frac{L\gamma_t^2}{2n}\sigma^2 + \frac{L\gamma_t^2}{2}\mathbb{E}\|\overline{\nabla f}(X_t)\|^2
	+  \frac{L}{2}\mathbb{E}\|\overline{Q}_t\|^2  \\
= & \mathbb{E}f(\overline{X}_t)  -  \frac{\gamma_t}{2}\mathbb{E}\|\nabla f(\overline{X}_t)\|^2  -  \frac{\gamma_t}{2}\mathbb{E}\|\overline{\nabla f}(X_t)\|^2  +  
	\frac{\gamma_t}{2}\mathbb{E}\|\nabla f(\overline{X}_t)  -  		
	\overline{\nabla f}(X_t)\|^2\\
	&  +  \frac{L\gamma_t^2}{2}\mathbb{E}\|\overline{\nabla f}(X_t)\|^2 + \frac{L}{2}\mathbb{E}\|\overline{Q}_t\|^2  +   \frac{L\gamma_t^2}{2n}\sigma^2\quad \text{(due to $2\langle \bm{a},\bm{b}\rangle=\|\bm{a}\|^2+\|\bm{b}\|^2-\|\bm{a}-\bm{b}\|^2$)}\\
= &  \mathbb{E}f(\overline{X}_t)  -  \frac{\gamma_t}{2}\mathbb{E}\|\nabla f(\overline{X}_t)\|^2  -  (\frac{\gamma_t}{2}-\frac{L\gamma_t^2}{2})\mathbb{E}\|\overline{\nabla f}(X_t)\|^2  +  
	\frac{\gamma_t}{2}\mathbb{E}\|\nabla f(\overline{X}_t)  -  		
	\overline{\nabla f}(X_t)\|^2 \\
		& + \frac{L}{2}\mathbb{E}\|\overline{Q}_t\|^2  +   \frac{L\gamma_t^2}{2n}\sigma^2. \numberthis \label{lemma:boudnfplus_long}    
\end{align*}
To estimate the upper bound for
$\mathbb{E}\|\nabla f(\overline{X}_t)  -  \overline{\nabla f}(X_t)\|^2$, we have 
\begin{align*}
\mathbb{E}\|\nabla f(\overline{X}_t)  -  \overline{\nabla f}(X_t)\|^2 = & \frac{1}{n^2}\mathbb{E}{\left\|
	\sum_{i=1}^n\left(\nabla f_i(\overline{X}_t)  -  \nabla f_i(\bm{x}_t^{(i)})\right)\right\|^2}\\
\leq & \frac{1}{n}\sum_{i=1}^n\mathbb{E}\left\|\nabla f_i(\overline{X}_t)  -  \nabla f_i(\bm{x}_t^{(i)})\right\|^2\\
\leq & \frac{L^2}{n}\mathbb{E}\sum_{i=1}^n\left\|\overline{X}_t-\bm{x}_t^{(i)}\right\|^2.\numberthis \label{lemma:boudnfplus_short}
\end{align*}
Combining \eqref{lemma:boudnfplus_long} and \eqref{lemma:boudnfplus_short} together, we have
\begin{align*}
\frac{\gamma_t}{2}\mathbb{E}\|\nabla f(\overline{X}_t)\|^2  +  (\frac{\gamma_t}{2}-\frac{L\gamma_t^2}{2})\mathbb{E}\|\overline{\nabla f}(X_t)\|^2  \leq & \mathbb{E}f(\overline{X}_t)  -  \mathbb{E}f(\overline{X}_{t+1})  +  
	\frac{L^2\gamma_t}{2n}\mathbb{E}\sum_{i=1}^n\left\|\overline{X}_t-\bm{x}_t^{(i)}\right\|^2 \\
		& + \frac{L}{2}\mathbb{E}\|\overline{Q}_t\|^2  +   \frac{L\gamma_t^2}{2n}\sigma^2.
\end{align*}
which completes the proof.
\end{proof}

\end{lemma}

\section{Analysis for Algorithm~\ref{alg2}}\label{sec3}
In Algorithm~\ref{alg2}, we have
\begin{align*}
Z_t = X_t(W_t-I) -\gamma F\left(X_t;\xi_t\right).\numberthis \label{global_z_t}
\end{align*}
 We will prove that
\begin{align*}
\sum_{t=1}^{T}E_{qt}\|Q_t\|^2_F 
\leq & 2\alpha^2\left(\frac{2\mu^2(1+2\alpha^2)}{(1-\rho)^2 - 4\mu^2\alpha^2}+1\right)\sum_{t=1}^{T}\gamma_t^2\|G\left(X_t;\xi_t\right)\|^2_F,
\end{align*}
which leads to 
\begin{align*}
\sum_{i=1}^{n}\sum_{t=1}^{T}\left(1-3D_1L^2\gamma_t^2\right)^2\mathbb{E}\left\|\overline{X}_t-\bm{x}_t^{(i)}\right\|^2
\leq & 2nD_1(\sigma^2+3\zeta^2)\sum_{t=1}^T\gamma_t^2 + 6nD_1\sum_{t=1}^T\gamma_t^2\mathbb{E}\left\|\nabla f(\overline{X}_t)\right\|^2.\end{align*}
$D_1$, $\mu$ and $\rho$ are defined in Theorem~\ref{theo_2}.

\begin{lemma}\label{lemma_bound_Q_t_alg2}
Under Assumption~\ref{ass:global}, when using Algorithm~\ref{alg2}, we have
\begin{align*}
\sum_{t=1}^{T}E_{qt}\|Q_t\|^2_F 
\leq & 2\alpha^2\left(\frac{2\mu^2(1+2\alpha^2)}{(1-\rho)^2 - 4\mu^2\alpha^2}+1\right)\sum_{t=1}^{T}\gamma_t^2\|G\left(X_t;\xi_t\right)\|^2_F
\end{align*}
when $(1-\rho)^2 - 4\mu^2\alpha^2>0$, where
\begin{align*}
\mu = & \max_{i\in\{2\cdots n\}}|\lambda_i-1|
\end{align*}
\end{lemma}
\begin{proof}
In the proof below, we use $[A]^{(i,j)}$ to indicate the $(i,j)$ element of matrix $A$.

For the noise induced by quantization, we have
\begin{align*}
\bm{r}_t^{(i)} = R_t\bm{e}^{(i)} =  Q_tP\bm{e}^{(i)},
\end{align*}
so
\begin{align*}
\|\bm{r}_t^{(i)}\|^2 = & \bm{e}^{(i)\top}P^{\top}Q_t^{\top}Q_tP\bm{e}^{(i)}\\
= & \left(\bm{v}^{(i)}\right)^{\top}Q_t^{\top}Q_t\bm{v}^{(i)}.
\end{align*}
As for $Q_t^{\top}Q_t$, the expectation of non-diagonal elements would be zero because the compression noise on node $i$ is independent on node $j$, which leads to 
\begin{align*}
\mathbb{E}_{q_t}\left[Q_t^{\top}Q_t\right]^{(i,j)} = & \mathbb{E}_{q_t}\sum_{k=1}^{N}Q_t^{(k,i)}Q_t^{(k,j)}\\
= & \tau_{ij}\sum_{k=1}^{N}\mathbb{E}_{q_t}\left(Q_t^{(k,i)}\right)^2, \quad \text{(due to $\mathbb{E}_{q_t} Q_t^{(k,i)} = 0$ for $\forall i \in \{1\cdots n\}$)}
\end{align*}
where $\tau_{ij} = 1$ if $i=j$, else $\tau_{ij}=0$.
Then
\begin{align*}
\mathbb{E}_{q_t}\|\bm{r}_t^{(i)}\|^2 = & \mathbb{E}_{q_t}\left(\bm{v}^{(i)}\right)^{\top}Q_t^{\top}Q_t\bm{v}^{(i)}\\
= & \sum_{j=1}^{n}\sum_{k=1}^{N}\left(\bm{v}_j^{(i)}\right)^2\mathbb{E}_{q_t}\left(Q_t^{(k,j)}\right)^2, 
\end{align*}
where $\bm{v}_j^{(i)}$ is the $j$th element of $\bm{v}^{(i)}$. So we have
\begin{align*}
\sum_{i=2}^n\mathbb{E}_{q_t}\|\bm{r}_t^{(i)}\|^2 = & \sum_{i=2}^n\sum_{j=1}^{n}\sum_{k=1}^{N}\left(\bm{v}_j^{(i)}\right)^2\mathbb{E}_{q_t}\left(Q_t^{(k,j)}\right)^2\\
\leq & \sum_{i=1}^n\sum_{j=1}^{n}\sum_{k=1}^{N}\left(\bm{v}_j^{(i)}\right)^2\mathbb{E}_{q_t}\left(Q_t^{(k,j)}\right)^2\\
\leq & \sum_{j=1}^{n}\sum_{k=1}^{N}\mathbb{E}_{q_t}\left(Q_t^{(k,j)}\right)^2 \quad \text{(due to $\sum_{i=1}^n\left(\bm{v}_j^{(i)}\right)^2 = 1$)}\\
= & \mathbb{E}_{q_t}\left\|Q_t\right\|^2_F. \numberthis \label{bound_r_t^2}
\end{align*}
$\mathbb{E}_{q_t}\left(Q_t^{(k,j)}\right)^2$ is the noise brought by quantization, which satisfies
\begin{align*}
\mathbb{E}_{q_t}\left(Q_t^{(k,j)}\right)^2 \leq & \alpha^2 \left(z^{(k,j)}_t\right)^2\\
= & \alpha^2\left(\left[X_t(W-I)\right]^{(k,j)} + [G\left(X_t;\xi_t\right)]^{(k,j)}\right)^2 \quad \text{(due to \eqref{global_z_t})} \\
= & 2\alpha^2\left(\left[X_t(W-I)\right]^{(k,j)}\right)^2 + 2\alpha^2\left([G\left(X_t;\xi_t\right)]^{(k,j)}\right)^2\\
= & 2\alpha^2\left(\left[X_tP(\Lambda-I)P^{\top}\right]^{(k,j)}\right)^2 + 2\alpha^2\left([G\left(X_t;\xi_t\right)]^{(k,j)}\right)^2\\
= & 2\alpha^2\left(\left[Y(\Lambda-I)P^{\top}\right]^{(k,j)}\right)^2 + 2\alpha^2\left([G\left(X_t;\xi_t\right)]^{(k,j)}\right)^2, 
\end{align*}
then
\begin{align*}
\mathbb{E}_{q_t}\left\|Q_t\right\|^2_F \leq & 2\alpha^2\left\|Y(\Lambda-I)P^{\top}\right\|^2_F + 2\alpha^2\left\|G\left(X_t;\xi_t\right)\right\|^2_F\\
= & 2\alpha^2\left\|Y(\Lambda-I)\right\|^2_F + 2\alpha^2\left\|G\left(X_t;\xi_t\right)\right\|^2_F\\
= & 2\alpha^2\left\|Y \begin{pmatrix}
 &0,&0,&\cdots,&0\\
 &0,&\lambda_2-1,&\cdots,&0\\
 &\hdotsfor{4}\\
 &0,&0,&\cdots,&\lambda_n-1
 \end{pmatrix}
 \right\|^2_F + 2\alpha^2\left\|G\left(X_t;\xi_t\right)\right\|^2_F \\
 \leq & 2\alpha^2\sum_{i=2}^n(\lambda_i -1)^2\|\bm{y}_t^{(i)}\|^2 + 2\alpha^2\left\|G\left(X_t;\xi_t\right)\right\|^2_F\\
\leq & 2\alpha^2\mu^2\sum_{i=2}^n\left\|\bm{y}_t^{(i)}\right\|^2 + 2\alpha^2\left\|G\left(X_t;\xi_t\right)\right\|^2_F. \quad \text{(due to $\mu = \max_{i\in\{2\cdots n\}}|\lambda_i-1|$)} \numberthis \label{alg2:Q_t^2}
\end{align*}
Together with \eqref{bound_r_t^2}, it comes to
\begin{align*}
\sum_{i=2}^n\mathbb{E}_{q_t}\left\|\bm{r}_t^{(i)}\right\|^2 \leq & 2\alpha^2\mu^2\sum_{i=2}^n\left\|\bm{y}_t^{(i)}\right\|^2 + 2\alpha^2\left\|G\left(X_t;\xi_t\right)\right\|^2_F. \numberthis \label{bound_r_t^2_second}
\end{align*}
From \eqref{y_t}, we have
\begin{align*}
\bm{y}_t^{(i)} = & \sum_{s=1}^{t-1}\lambda_i^{t-s-1}\left(-\bm{h}_s^{(i)}+\bm{r}_s^{(i)}\right),\\
\left\|\bm{y}_t^{(i)}\right\|^2 \leq & \left(\sum_{s=1}^{t-1}|\lambda_i|^{t-s-1}\left\|-\bm{h}_s^{(i)}+\bm{r}_s^{(i)}\right\|\right)^2.\quad \text{(due to $\|\bm{a} + \bm{b}\|^2 \leq \|\bm{a}\|^2 + \|\bm{b}\|^2$)}
\end{align*}
Denote $m_t^{(i)} = \sum_{s=1}^{t-1}|\lambda_i|^{t-1-s}\left\|-\bm{h}_s^{(i)}+\bm{r}_s^{(i)}\right\|$, we can see that $m_t^{(i)}$ has the same structure of the sequence in Lemma~\ref{lemma1}, Therefore
\begin{align*}
\sum_{i=2}^n\sum_{t=1}^{T}\left(m_t^{(i)}\right)^2\leq & \sum_{i=2}^n\sum_{t=1}^{T}\frac{1}{(1-|\lambda_i|)^2}\left\|-\bm{h}_t^{(i)}+\bm{r}_t^{(i)}\right\|^2\\
\leq & \sum_{i=2}^n\sum_{t=1}^{T}\frac{2}{(1-|\lambda_i|)^2}\left(\left\|\bm{h}_t^{(i)}\right\|^2 + \left\|\bm{r}_t^{(i)}\right\|^2\right)\\
= & \sum_{i=2}^n\sum_{t=1}^{T}\frac{2}{(1-|\lambda_i|)^2}\left(\gamma_t^2\left\|G\left(X_t;\xi_t\right)\bm{v}^{(i)}\right\|^2+\left\|\bm{r}_t^{(i)}\right\|^2\right)\\
= & \sum_{i=2}^n\sum_{t=1}^{T}\frac{2\gamma_t^2\left\|G\left(X_t;\xi_t\right)\bm{v}^{(i)}\right\|^2}{(1-|\lambda_i|)^2} + \sum_{i=2}^n\sum_{t=1}^{T}\frac{2}{(1-|\lambda_i|)^2}\left\|\bm{r}_t^{(i)}\right\|^2\\
\leq & \sum_{i=2}^n\sum_{t=1}^{T}\frac{2\gamma_t^2\left\|G\left(X_t;\xi_t\right)\bm{v}^{(i)}\right\|^2}{(1-|\lambda_i|)^2} + \sum_{i=2}^n\sum_{t=1}^{T}\frac{2}{(1-|\lambda_i|)^2}\left\|\bm{r}_t^{(i)}\right\|^2\\
\leq & \frac{2}{(1-\rho)^2}\sum_{i=2}^n\sum_{t=1}^{T}\gamma_t^2\left\|G\left(X_t;\xi_t\right)\bm{v}^{(i)}\right\|^2 + \frac{2}{(1-\rho)^2}\sum_{i=2}^n\sum_{t=1}^{T}\left\|\bm{r}_t^{(i)}\right\|^2 \\
\leq & \frac{2}{(1-\rho)^2}\sum_{i=2}^n\sum_{t=1}^{T}\gamma_t^2\left\|G\left(X_t;\xi_t\right)P\right\|^2 + \frac{2}{(1-\rho)^2}\sum_{i=2}^n\sum_{t=1}^{T}\left\|\bm{r}_t^{(i)}\right\|^2 \\
\leq & \frac{2}{(1-\rho)^2}\sum_{t=1}^{T}\gamma_t^2\left\|G\left(X_t;\xi_t\right)\right\|^2_F + \frac{2}{(1-\rho)^2}\sum_{i=2}^n\sum_{t=1}^{T}\left\|\bm{r}_t^{(i)}\right\|^2 .\numberthis \label{bound_y_t_second}
\end{align*}
Combining \eqref{bound_r_t^2_second} and \eqref{bound_y_t_second} together, we have
\begin{align*}
\sum_{i=2}^n\sum_{t=1}^T\left\|\bm{y}_t^{(i)}\right\|^2 \leq\sum_{i=2}^n\sum_{t=1}^{T}\left(m_t^{(i)}\right)^2 
\leq & \frac{2(1+2\alpha^2)}{(1-\rho)^2}\sum_{t=1}^{T}\gamma_t^2\left\|G\left(X_t;\xi_t\right)\right\|^2_F + \frac{4\mu^2\alpha^2}{(1-\rho)^2}\sum_{l=2}^n\sum_{t=1}^T\left\|\bm{y}_t^{(l)}\right\|^2,
\end{align*}
It follows thtat
\begin{align*}
\sum_{i=2}^n\sum_{t=1}^T\left\|\bm{y}_t^{(i)}\right\|^2 \leq & \frac{2(1+2\alpha^2)}{(1-\rho)^2}\sum_{t=1}^{T}\gamma_t^2\left\|G\left(X_t;\xi_t\right)\right\|^2_F + \frac{4\mu^2\alpha^2}{(1-\rho)^2}\sum_{t=1}^T\sum_{l=2}^n\|\bm{y}_t^{(l)}\|^2,\\
\left(1-\frac{4\mu^2\alpha^2}{(1-\rho)^2}\right)\sum_{i=2}^n\sum_{t=1}^T\left\|\bm{y}_t^{(i)}\right\|^2 \leq & \frac{2(1+2\alpha^2)}{(1-\rho)^2}\sum_{t=1}^{T}\gamma_t^2\left\|G\left(X_t;\xi_t\right)\right\|^2_F.
\end{align*}
If $\alpha$ is small enough that satisfies $(1-\rho)^2 - 4\mu^2\alpha^2>0$, then we have
\begin{align*}
\sum_{i=2}^n\sum_{t=1}^T\left\|\bm{y}_t^{(i)}\right\|^2 \leq & \frac{2(1+2\alpha^2)}{(1-\rho)^2 - 4\mu^2\alpha^2}\sum_{t=1}^{T}\gamma_t^2\left\|G\left(X_t;\xi_t\right)\right\|^2_F.\numberthis \label{bound_y_t_third}
\end{align*}
Together \eqref{bound_y_t_third} with \eqref{alg2:Q_t^2} , we have
\begin{align*}
\sum_{t=1}^{T}E_{qt}\|Q_t\|^2_F \leq & 2\alpha^2\mu^2\sum_{t=1}^T\sum_{l=2}^{n}\|\bm{y}_t^{(l)}\|^2 + 2\alpha^2\sum_{t=1}^{T}\gamma_t^2\left\|G\left(X_t;\xi_t\right)\right\|^2_F\\
\leq & 2\alpha^2\left(\frac{2\mu^2(1+2\alpha^2)}{(1-\rho)^2 - 4\mu^2\alpha^2}+1\right)\sum_{t=1}^{T}\gamma_t^2\|G\left(X_t;\xi_t\right)\|^2_F.
\end{align*}
{\rc Moreover, setting $\gamma_t = \gamma$ and denote $2\alpha^2\left(\frac{2\mu^2(1+2\alpha^2)}{(1-\rho)^2 - 4\mu^2\alpha^2}+1\right) = D_2$. Applying Lemma~\ref{lemma3} to bound $\|G\left(X_t;\xi_t\right)\|^2_F$, we have}
\begin{align*}
\sum_{t=1}^{T}E_{qt}\|Q_t\|^2_F \leq & n\sigma^2D_2\gamma^2T+3L^2D_2\gamma^2\sum_{i=1}^n\mathbb{E}\left\|\overline{X}_t-\bm{x}_t^{(i)}\right\|^2+3n\zeta^2D_2\gamma^2T\\
&+3nD_2\gamma^2\mathbb{E}\sum_{t=1}^T\left\|\nabla f(\overline{X}_t)\right\|^2.\numberthis \label{alg2_mymiss1}
\end{align*}
\end{proof}

\begin{lemma}\label{lemma_bound_X_t_alg2}
Under Assumption~\ref{ass:global}, we have
\begin{align*}
\sum_{i=1}^{n}\sum_{t=1}^{T}\left(1-3D_1L^2\gamma_t^2\right)^2\mathbb{E}\left\|\overline{X}_t-\bm{x}_t^{(i)}\right\|^2
\leq & 2nD_1(\sigma^2+3\zeta^2)\sum_{t=1}^T\gamma_t^2 + 6nD_1\sum_{t=1}^T\gamma_t^2\mathbb{E}\left\|\nabla f(\overline{X}_t)\right\|^2.
\end{align*}
when $1-2\alpha^2\mu^2C_{\lambda}>0$, where 
\begin{align*}
\mu = & \max_{l\in\{2\cdots n\}}|\lambda_l-1|\\
\rho = & \max_{l\in\{2\cdots n\}}|\lambda_l|\\
D_1 = & \frac{2\alpha^2}{1-\rho^2}\left(\frac{2\mu^2(1+2\alpha^2)}{(1-\rho)^2 - 4\mu^2\alpha^2}+1\right) + \frac{1}{(1-\rho)^2}.
\end{align*}
\end{lemma}

\begin{proof}
From Lemma~\ref{lemma_bound_all_X_ave}, we have
\begin{align*}
&\sum_{t=1}^{T}\sum_{i=1}^{n}\mathbb{E}\left\|\overline{X}_t-\bm{x}_t^{(i)}\right\|^2\\
\leq & \frac{2}{1-\rho^2}\sum_{t=1}^{T}\|Q_s\|^2_F + 2\frac{1}{(1-\rho)^2}\sum_{t=1}^{T}\gamma_t^2\|G(X_{t};\xi_{t})\|^2_F\\
\leq & 2\left(\frac{2\alpha^2}{1-\rho^2}\left(\frac{2\mu^2(1+2\alpha^2)}{(1-\rho)^2 - 4\mu^2\alpha^2}+1\right) + \frac{1}{(1-\rho)^2}\right)\sum_{t=1}^T\gamma_t^2\|G(X_{t};\xi_{t})\|^2. \quad \text{(due to Lemma~\ref{lemma_bound_Q_t_alg2})}
\end{align*}
Denote $\frac{2\alpha^2}{1-\rho^2}\left(\frac{2\mu^2(1+2\alpha^2)}{(1-\rho)^2 - 4\mu^2\alpha^2}+1\right) + \frac{1}{(1-\rho)^2} = D_1$,then from Lemma~\ref{lemma3}, we have
\begin{align*}
\sum_{t=1}^{T}\sum_{i=1}^{n}\mathbb{E}\left\|\overline{X}_t-\bm{x}_t^{(i)}\right\|^2
\leq & 2nD_1(\sigma^2+3\zeta^2)\sum_{t=1}^T\gamma_t^2 + 6nD_1\sum_{t=1}^T\gamma_t^2\mathbb{E}\left\|\nabla f(\overline{X}_t)\right\|^2\\
& + 3D_1L^2\sum_{i=1}^n\gamma_t^2\mathbb{E}\left\|\overline{X}_t-\bm{x}_t^{(i)}\right\|^2,\\
 \sum_{t=1}^{T}\left(1-3D_1L^2\gamma_t^2\right)^2\sum_{i=1}^{n}\mathbb{E}\left\|\overline{X}_t-\bm{x}_t^{(i)}\right\|^2
\leq & 2nD_1(\sigma^2+3\zeta^2)\sum_{t=1}^T\gamma_t^2 + 6nD_1\sum_{t=1}^T\gamma_t^2\mathbb{E}\left\|\nabla f(\overline{X}_t)\right\|^2.
\end{align*}
If $1-3D_1L^2\gamma_t^2 > 0$, then $\sum_{t=1}^{T}\sum_{i=1}^{n}\mathbb{E}\left\|\overline{X}_t-\bm{x}_t^{(i)}\right\|^2$ would be bounded.
\end{proof}

{\bf \noindent Proof to Theorem \ref{theo_2}}
\begin{proof}
Combining Lemma~\ref{lemma:boundfplus} and \eqref{alg2_mymiss1}, we have
\begin{align*}
\mathbb{E}\|\nabla f(\overline{X}_t)\|^2 + \left(1-L\gamma\right)\mathbb{E}\|\overline{\nabla f}(X_t)\|^2
\leq & \frac{2}{\gamma}\left(\mathbb{E}f(\overline{X}_t)-f^*-\left(\mathbb{E}f(\overline{X}_{(t+1)})-f^*\right)\right)\\
	& + \left(\frac{2L^2}{n} + \frac{3L^3D_2\gamma^2}{2n}\right)\sum_{i=1}^{n}\|\overline{X}_t-x_t^{(i)}\|^2
	\\
    & + \left(\frac{\gamma^2LD_2}{2} + \frac{L\gamma}{n}\right)T\sigma^2 + \frac{3LD_2\gamma^2\zeta^2T}{2}\\
    & + \frac{3LD_2\gamma^2}{2}\mathbb{E}\left\|\nabla f(\overline{X}_t)\right\|^2. \numberthis \label{bound_nabla_f_alg2_second}
\end{align*}
From Lemma~\ref{lemma_bound_substract_mean}, we have
\begin{align*}
\left(1-3D_1L^2\gamma^2\right)^2\sum_{t=1}^{T}\sum_{i=1}^{n}\mathbb{E}\left\|\overline{X}_t-\bm{x}_t^{(i)}\right\|^2
\leq & 2nD_1(\sigma^2+3\zeta^2)T\gamma^2 + 6nD_1\gamma^2\sum_{t=1}^T\mathbb{E}\left\|\nabla f(\overline{X}_t)\right\|^2.
\end{align*}
If $\gamma$ is not too large that satisfies $1-3D_1L^2\gamma^2 > 0$, we have
\begin{align*}
\sum_{t=1}^{T}\sum_{i=1}^{n}\mathbb{E}\left\|\overline{X}_t-\bm{x}_t^{(i)}\right\|^2
\leq & \frac{2nD_1(\sigma^2+3\zeta^2)T\gamma^2}{1-3D_1L^2\gamma^2} + \frac{6nD_1\gamma^2}{1-3D_1L^2\gamma^2}\sum_{t=1}^T\mathbb{E}\left\|\nabla f(\overline{X}_t)\right\|^2. \numberthis \label{bound_xmean_alg2}
\end{align*}
 Summarizing both sides of \eqref{bound_nabla_f_alg2_second} and applying \eqref{bound_xmean_alg2} yields
\begin{align*}
& \sum_{t=1}^T\left(\mathbb{E}\|\nabla f(\overline{X}_t)\|^2 + \left(1-L\gamma\right)\mathbb{E}\|\overline{\nabla f}(X_t)\|^2\right) \\
\leq & \frac{2(f(0)-f^*)}{\gamma} +
\left(\frac{T\gamma^2LD_2}{2} + \frac{TL\gamma}{n} + \frac{\left(4L^2 + 3L^3D_2\gamma^2\right)D_1T\gamma^2}{1-3D_1L^2\gamma^2}\right)\sigma^2 \\
& + \frac{\left(4L^2 + 3L^3D_2\gamma^2\right)3D_1T\gamma^2}{1-3D_1L^2\gamma^2}\zeta^2 + \frac{3LD_2\gamma^2T}{2}\zeta^2\\
& + \left(\frac{\left(4L^2 + 3L^3D_2\gamma^2\right)3D_1\gamma^2}{1-3D_1L^2\gamma^2} + \frac{3LD_2\gamma^2}{2} \right)\sum_{t=1}^T\mathbb{E}\left\|\nabla f(\overline{X}_t)\right\|^2.
\end{align*}
It implies
\begin{align*}
& \sum_{t=1}^T\left(\left(1-\frac{\left(4L^2 + 3L^3D_2\gamma^2\right)3D_1\gamma^2}{1-3D_1L^2\gamma^2} - \frac{3LD_2\gamma^2}{2}\right)\mathbb{E}\|\nabla f(\overline{X}_t)\|^2 + \left(1-L\gamma\right)\mathbb{E}\|\overline{\nabla f}(X_t)\|^2\right)
\\ 
\leq & \frac{2(f(0)-f^*)}{\gamma} +
\left(\frac{T\gamma^2LD_2}{2} + \frac{L\gamma T}{n} + \frac{\left(4L^2 + 3L^3D_2\gamma^2\right)D_1T\gamma^2}{1-3D_1L^2\gamma^2}\right)\sigma^2 \\
& + \left( \frac{\left(4L^2 + 3L^3D_2\gamma^2\right)3D_1T\gamma^2}{1-3D_1L^2\gamma^2} + \frac{3LD_2\gamma^2T}{2} \right)\zeta^2.
\end{align*}
Denote $D_3 = \frac{\left(4L^2 + 3L^3D_2\gamma^2\right)3D_1\gamma^2}{1-3D_1L^2\gamma^2} + \frac{3LD_2\gamma^2}{2}$ and $D_4 = 1-L\gamma$, we have
\begin{align*}
& \sum_{t=1}^T\left(\left(1-D_3\right)\mathbb{E}\|\nabla f(\overline{X}_t)\|^2 + D_4\mathbb{E}\|\overline{\nabla f}(X_t)\|^2\right)
\\ 
\leq & \frac{2(f(0)-f^*)}{\gamma} +
\left(\frac{T\gamma^2LD_2}{2} + \frac{L\gamma T}{n} + \frac{\left(4L^2 + 3L^3D_2\gamma^2\right)D_1T\gamma^2}{1-3D_1L^2\gamma^2}\right)\sigma^2 \\
& + \left( \frac{\left(4L^2 + 3L^3D_2\gamma^2\right)3D_1T\gamma^2}{1-3D_1L^2\gamma^2} + \frac{3LD_2\gamma^2T}{2} \right)\zeta^2.
\end{align*}
It completes the proof.
\end{proof}

{\bf \noindent Proof to Corollary \ref{cor:convergence_alg2}}
\begin{proof}
Setting $\gamma = \frac{1}{6\sqrt{D_1}L + 6\sqrt{D_2L}+\frac{\sigma}{\sqrt{n}}T^{\frac{1}{2}} + \zeta^{\frac{2}{3}}T^{\frac{1}{3}}}$, then we can verify
\begin{align*}
3D_1L^2\gamma^2 \leq & \frac{1}{12}\\
3LD_2\gamma^2 \leq & \frac{1}{12}\\
D_3 \leq & \frac{1}{2}\\
D_4 \geq & 0.
\end{align*}
So we can remove the $\|\overline{\nabla f}(X_t)\|^2$ on the LHS and substitute $(1-D_3)$ with $\frac{1}{2}$. Therefore \eqref{bound_theo_2} becomes
\begin{align*}
\frac{1}{T}\sum_{t=1}^T\mathbb{E}\|\nabla f(\overline{X}_t)\|^2 \leq & 4(f(0)-f^*)\frac{\sigma}{\sqrt{Tn}} + \frac{4L\sigma}{\sqrt{Tn}}  + \frac{\zeta^{\frac{2}{3}}}{T^{\frac{2}{3}}}(4LD_2 + 30L^2D_1 + 4(f(0) - f^*))\\
& + \frac{n\sigma^2}{nD_1L^2 + \sigma^2T}\left(5D_1L^2 + \frac{LD_2}{2}\right) + \frac{4(f(0) - f^*)(6\sqrt{D_1}L + 6\sqrt{D_2L})}{T}.\numberthis\label{alg2_pro_coro_1}
\end{align*}
{\nrc From Lemma~\ref{lemma_bound_X_t_alg2}, we have
\begin{align*}
\frac{1}{T}\sum_{i=1}^{n}\sum_{t=1}^{T}\left(1-3D_1L^2\gamma^2\right)^2\mathbb{E}\left\|\overline{X}_t-\bm{x}_t^{(i)}\right\|^2
\leq & 2nD_1(\sigma^2+3\zeta^2)\gamma^2 + 6nD_1\gamma^2\frac{1}{T}\sum_{t=1}^T\mathbb{E}\left\|\nabla f(\overline{X}_t)\right\|^2\\
\leq& \frac{2n\sqrt{n}D_1}{T} + \frac{6n\zeta^{\frac{2}{3}}}{T^{\frac{2}{3}}} + 6nD_1\gamma^2\frac{1}{T}\sum_{t=1}^T\mathbb{E}\left\|\nabla f(\overline{X}_t)\right\|^2.\\
\frac{1}{T}\sum_{i=1}^{n}\sum_{t=1}^{T}\mathbb{E}\left\|\overline{X}_t-\bm{x}_t^{(i)}\right\|^2
\leq & \frac{4n\sqrt{n}D_1}{T} + \frac{12n\zeta^{\frac{2}{3}}}{T^{\frac{2}{3}}} + 12nD_1\gamma^2\frac{1}{T}\sum_{t=1}^T\mathbb{E}\left\|\nabla f(\overline{X}_t)\right\|^2.\numberthis\label{alg2_pro_coro_2}
\end{align*}
If $\alpha^2 \leq \min\left\{\frac{(1-\rho)^2}{8\mu^2},\frac{1}{4}\right\}$, then
\begin{align*}
D_2 = & 2\alpha^2\left(\frac{2\mu^2(1+2\alpha^2)}{(1-\rho)^2 - 4\mu^2\alpha^2}+1\right)
\leq  \frac{3\mu^2\alpha^2}{(1-\rho)^2} + 2\alpha^2,\\
D_1 = & \frac{D_2}{1-\rho^2} + \frac{1}{(1-\rho)^2}
\leq \frac{D_2 + 1}{(1-\rho)^2}, \quad \text{(due to $\rho < 1$)}
\end{align*}
which means $D_2 = O\left(\alpha^2\right) $ and $D_1 = O\left(\alpha^2 +1\right)$.\\
So \eqref{alg2_pro_coro_2} becomes
\begin{align*}
\frac{1}{T}\sum_{t=1}^T\mathbb{E}\|\nabla f(\overline{X}_t)\|^2 \lesssim & \frac{\sigma}{\sqrt{nT}} + \frac{\zeta^{\frac{2}{3}}(1+\alpha)}{T^{\frac{2}{3}}}+ \frac{1+\alpha^2}{T}.
\end{align*}
Combing the inequality above and \eqref{alg2_pro_coro_2}, we have
\begin{align*}
\frac{1}{T}\sum_{i=1}^{n}\sum_{t=1}^{T}\mathbb{E}\left\|\overline{X}_t-\bm{x}_t^{(i)}\right\|^2
\lesssim &  \frac{n\zeta^{\frac{2}{3}}}{T^{\frac{2}{3}}} + \frac{n\sqrt{n}(1+\alpha^2)}{T}+ \frac{n(1+\alpha)}{T^2}.\numberthis\label{finalcoro_1}
\end{align*}

}
\end{proof}

\section{Analysis for Algorithm~\ref{alg1}}\label{sec2}
We are going to prove that by using \eqref{alg1_noisecontrol1} and \eqref{alg1_noisecontrol2} in Algorithm~\ref{alg1}, the upper bound for the compression noise would be
\begin{align*}
\mathbb{E}\|Q_t\|^2_F\leq \frac{n\tilde{\sigma}^2}{t}.
\end{align*}
Therefore, combing this with Lemma~\ref{lemma:boundfplus} and Lemma~\ref{lemma_bound_all_X_ave},  we would be able to prove the convergence rate for ALgorithm~\ref{alg1}.
\begin{align*}
\sum_{t=1}^{T}\sum_{i=1}^n\left(1-6nC_1L^2\gamma_t^2\right)\mathbb{E}\left\|\overline{X}_t-x_t^{(i)}\right\|^2
\leq &\frac{2n^2\tilde{\sigma}^2}{1-\rho^2}\log T + 2\left(\sigma^2+3\zeta^2\right)n^2C_1\sum_{t=1}^{T-1}\gamma_t^2\\
&+6C_1n^2\sum_{t=1}^{T-1}\mathbb{E}\gamma_t^2\left\|\nabla f(\overline{X}_t)\right\|^2,
\end{align*}
which ensures that all nodes would converge to the same value.

\begin{lemma}\label{lemma:add}
For any non-negative sequences $\{a_n\}_{n=1}^{+\infty}$ and $\{b_n\}_{n=1}^{+\infty}$ that satisfying
\begin{align*}
& a_t=  \left(1-\frac{2}{t}\right)^2a_{t-1}+\frac{4}{t^2}b_t\\
& b_t\leq  \frac{\tilde{\sigma}^2}{2}\quad \forall t\in\{1,2,3,\cdots\}\\
& a_1=0,
\end{align*}
we have
\begin{align*}
a_t\leq\frac{\tilde{\sigma}^2}{t}{\color{red}.}
\end{align*}
\begin{proof}
We use induction to prove the lemma. Since $a_1=0\leq \tilde{\sigma}^2$, suppose the lemma holds for $t\leq k$, which means $a_t\leq\frac{\tilde{\sigma}^2}{t}$, for $\forall t\leq k$. Then it leads to
\begin{align*}
a_{k+1}= & \left(1-\frac{2}{k}\right)^2a_{k}+\frac{4}{k^2}b_t\\
\leq & \left(1-\frac{2}{k}\right)^2a_{k}+\frac{2\tilde{\sigma}^2}{k^2}\\
\leq & \left(1-\frac{4}{k}+\frac{4}{k^2}\right)\frac{\tilde{\sigma}^2}{k} + \frac{2\tilde{\sigma}^2}{k^2}\\
= & \tilde{\sigma}^2\left(\frac{1}{k}-\frac{1}{k+1}\right)-\frac{2\tilde{\sigma}^2}{k^2}+\frac{4\tilde{\sigma}^2}{k^3}+ \frac{\tilde{\sigma}^2}{k+1}\\
= & \frac{\tilde{\sigma}^2}{k^2(k+1)}\left(k-2(k+1)+\frac{4}{k}\right)+\frac{\tilde{\sigma}^2}{k+1}\\
= & \frac{\tilde{\sigma}^2}{k^2(k+1)}\left(-k-2+\frac{4}{k}\right)+\frac{\tilde{\sigma}^2}{k+1}\\
\leq & \frac{\tilde{\sigma}^2}{k+1}.\quad\text{(due to $k\geq 2$)}
\end{align*}
It completes the proof.
\end{proof}
\end{lemma}

\begin{lemma} \label{lemma3}
Under the Assumption \ref{ass:global}, when using Algorithm~\ref{alg1}, we have
\begin{align*}
\mathbb{E}\left\|G(X_t,\xi_t)\right\|^2_F\leq &  n\sigma^2+3L^2\sum_{i=1}^n\mathbb{E}\left\|\overline{X}_t-\bm{x}_t^{(i)}\right\|^2+3n\zeta^2+3n\mathbb{E}\left\|\nabla f(\overline{X}_t)\right\|^2,\\
\mathbb{E}\left\|Q_t\right\|^2_F \leq &\frac{n\tilde{\sigma}^2}{t},\\
\mathbb{E}\left\|\overline{Q}_t\right\|^2_F \leq &\frac{\tilde{\sigma}^2}{nt}.
\end{align*}
\end{lemma}
\begin{proof}
Notice that
\begin{align*}
\mathbb{E}\left\|G(X_t,\xi_t)\right\|^2_F
= \sum_{i=1}^{n}\mathbb{E}\left\|\nabla F_i(x_t^{(i)};\xi_t^{(i)})\right\|^2.
\end{align*}
We next estimate the upper bound of $\mathbb{E}\left\|\nabla F_i(\bm{x}_t^{(i)};\xi_t^{(i)})\right\|^2$ in the following
\begin{align*}
\mathbb{E}\left\|\nabla F_i(\bm{x}_t^{(i)};\xi_t^{(i)})\right\|^2
= & \mathbb{E}\left\|\left(\nabla F_i(\bm{x}_t^{(i)};\xi_t^{(i)})-\nabla f_i(\bm{x}_t^{(i)})\right)+\nabla f_i(\bm{x}_t^{(i)})\right\|^2\\
= & \mathbb{E}\left\|\nabla F_i(\bm{x}_t^{(i)};\xi_t^{(i)})-\nabla f_i(\bm{x}_t^{(i)})\right\|^2+\mathbb{E}\left\|\nabla f_i(\bm{x}_t^{(i)})\right\|^2\\
 & + 2\mathbb{E}\left\langle\mathbb{E}_{\xi_t}\nabla F_i(\bm{x}_t^{(i)};\xi_t^{(i)})-\nabla f_i(\bm{x}_t^{(i)}),\nabla f_i(\bm{x}_t^{(i)})\right\rangle\\
= &\mathbb{E}\left\|\nabla F_i(\bm{x}_t^{(i)};\xi_t^{(i)})-\nabla f_i(\bm{x}_t^{(i)})\right\|^2+\mathbb{E}\left\|\nabla f_i(\bm{x}_t^{(i)})\right\|^2\\
\leq & \sigma^2 + \mathbb{E}\left\|\left(\nabla f_i(\bm{x}_t^{(i)})-\nabla f_i(\overline{X}_t)\right)+\left(\nabla f_i(\overline{X}_t)-\nabla f(\overline{X}_t)\right)+\nabla f(\overline{X}_t)\right\|^2\\
\leq & \sigma^2 + 3\mathbb{E}\left\|\nabla f_i(\bm{x}_t^{(i)})-\nabla f_i(\overline{X}_t)\right\|^2 + 3\mathbb{E}\left\|\nabla f_i(\overline{X}_t)-\nabla f(\overline{X}_t)\right\|^2\\
& + 3\mathbb{E}\left\|\nabla f(\overline{X}_t)\right\|^2\\
\leq & \sigma^2+3L^2\mathbb{E}\left\|X_t-\bm{x}_t^{(i)}\right\|^2+3\zeta^2+3\mathbb{E}\left\|\nabla f(\overline{X}_t)\right\|^2,
\end{align*}
which means
\begin{align*}
\mathbb{E}\left\|G(X_t,\xi_t)\right\|^2\leq
\sum_{i=1}^{n}\left\|\nabla F_i(\bm{x}_t^{(i)};\xi_t^{(i)})\right\|^2
\leq & n\sigma^2+3L^2\sum_{i=1}^n\mathbb{E}\left\|\overline{X}_t-\bm{x}_t^{(i)}\right\|^2+3n\zeta^2\\
&+3n\mathbb{E}\left\|\nabla f(\overline{X}_t)\right\|^2.
\end{align*}
From \eqref{alg1_noisecontrol2}, we have
\begin{align*}
\tilde{\bm{x}}^{(j)}_t - \bm{x}^{(j)}_t
= & (1-2t^{-1})(\tilde{\bm{x}}^{(j)}_{t-1} - \bm{x}^{(j)}_{t-1}) + 2t^{-1} \bm{q}^{(j)}_t,\numberthis\label{alg1_noise}
\end{align*}
Then we have
\begin{align*}
&\mathbb{E}\|\tilde{\bm{x}}^{(j)}_t - \bm{x}^{(j)}_t\|^2\\
= & \left(1-2t^{-1}\right)^2\mathbb{E}\|\tilde{\bm{x}}^{(j)}_{t-1} - \bm{x}^{(j)}_{t-1}\|^2+4t^{-2}\mathbb{E}\|\bm{q}_t^{(j)}\|^2
 + 2t^{-1}\left(1-2t^{-1}\right)\mathbb{E}\left\langle\tilde{\bm{x}}^{(j)}_{t-1} - \bm{x}^{(j)}_{t-1},\mathbb{E}_{\bm{q}_t^{(j)}}\bm{q}_t^{(j)}\right\rangle
\end{align*}
Since $\mathbb{E}_{q_t^{(j)}}\bm{C}(\bm{z}_t^{(j)}) = \bm{z}_t^{(j)}$, so we have $\mathbb{E}_{q_t^{(j)}}\bm{q}_t^{(j)} = 0$, then
\begin{align*}
\mathbb{E}\|\tilde{\bm{x}}^{(j)}_t - \bm{x}^{(j)}_t\|^2 = \left(1-2t^{-1}\right)^2\mathbb{E}\|\tilde{\bm{x}}^{(j)}_{t-1} - \bm{x}^{(j)}_{t-1}\|^2+4t^{-2}\mathbb{E}\|\bm{q_t^{(j)}}\|^2.\numberthis\label{noise_evalution}
\end{align*}

Meanwhile, \eqref{noise_evalution} indicates that
\begin{align*}
\mathbb{E}\left\|\bm{q}_t^{(i)}\right\|^2 \leq \left(1-\frac{2}{t}\right)^2\mathbb{E}\left\|\bm{q}_{t-1}^{(i)}\right\|^2 + \frac{4}{t^2}\frac{\tilde{\sigma}^2}{2}.\numberthis \label{control_noise_1}
\end{align*}
So applying Lemma~\ref{lemma3} into \eqref{control_noise_1}, we have
\begin{align*}
\mathbb{E}\left\|\bm{q}_t^{(i)}\right\|^2 \leq \frac{\tilde{\sigma}^2}{t}.
\end{align*}
Therefore
\begin{align*}
\mathbb{E}\left\|Q_t\right\|^2_F = & 
\sum_{i=1}^n\mathbb{E}\left\|\bm{q}_t^{(i)}\right\|^2
\leq \frac{n\tilde{\sigma}^2}{t}, \quad \left( \text{due to $\mathbb{E}\left\|\bm{q}_t^{(i)}\right\|^2\leq\frac{\tilde{\sigma}^2}{t}$}\right)\\
\mathbb{E}\left\|\overline{Q}_t\right\|^2_F = &
\frac{1}{n^2}\mathbb{E}\left\|\sum_{i=1}^n\bm{q}_t^{(i)}\right\|^2\\
= & \frac{1}{n^2}\sum_{i=1}^n\mathbb{E}\left\|\bm{q}_t^{(i)}\right\|^2+ \sum_{i\neq i'}^n\mathbb{E}\left\langle\bm{q}_t^{(i)},\bm{q}_t^{(i')}\right\rangle\\
\leq & \frac{\tilde{\sigma}^2}{nt}. \quad\left(\text{due to $\mathbb{E}\bm{q}_t^{(i)}=0$ for $\forall i\in \{1,\cdots,n\}$}\right)
\end{align*}
\end{proof}

\begin{lemma}\label{lemma:boundx}
Under Assumption \ref{ass:global}, when using Algorithm~\ref{alg1}, we have 
\begin{align*}
\sum_{t=1}^{T}\sum_{i=1}^n\left(1-6C_1L^2\gamma_t^2\right)\mathbb{E}\left\|\overline{X}_t-\bm{x}_t^{(i)}\right\|^2
\leq &\frac{2n\tilde{\sigma}^2}{1-\rho^2}\log T + 2\left(\sigma^2+3\zeta^2\right)nC_1\sum_{t=1}^{T-1}\gamma_t^2\\
&+6C_1n\sum_{t=1}^{T-1}\mathbb{E}\gamma_t^2\left\|\nabla f(\overline{X}_t)\right\|^2, \numberthis \label{theo:boundxtmode}
\end{align*}
{\rc where $C_1$ is defined in Theorem~\ref{theo_1}. }
\end{lemma}
\begin{proof}
From Lemma~\ref{lemma_bound_all_X_ave} and Lemma~\ref{lemma3}, we have
\begin{align*}
\sum_{i=1}^n\sum_{t=1}^T\mathbb{E}\left\|\overline{X}_t-\bm{x}_t^{(i)}\right\|^2 \leq &
\frac{2n\tilde{\sigma}^2}{1-\rho^2}\sum_{t=1}^{T-1}\frac{1}{t} + 2C_1\sum_{t=1}^{T-1}\gamma_t^2\mathbb{E}\left\|G(X_{t};\xi_{t})\right\|^2\\
\leq & \frac{2n\tilde{\sigma}^2}{1-\rho^2}\log T + 2\left(n\sigma^2+3n\zeta^2\right)C_1\sum_{t=1}^{T-1}\gamma_t^2+6C_1n\sum_{t=1}^{T-1}\gamma_t^2\mathbb{E}\left\|\nabla f(\overline{X}_t)\right\|^2\\
& + 6C_1L^2\sum_{t=1}^{T-1}\sum_{i=1}^n\gamma_t^2\mathbb{E}\left\|\overline{X}_t-\bm{x}_t^{(i)}\right\|^2. \quad\text{(due to Lemma~\ref{lemma3})}
\end{align*}
Rearranging it obtain the following
\begin{align*}
\sum_{t=1}^{T}\sum_{i=1}^n\left(1-6C_1L^2\gamma_t^2\right)\mathbb{E}\left\|\overline{X}_t-\bm{x}_t^{(i)}\right\|^2
\leq &\frac{2n\tilde{\sigma}^2}{1-\rho^2}\log T + 2\left(\sigma^2+3\zeta^2\right)nC_1\sum_{t=1}^{T-1}\gamma_t^2\\
&+6C_1n\sum_{t=1}^{T-1}\mathbb{E}\gamma_t^2\left\|\nabla f(\overline{X}_t)\right\|^2, \numberthis \label{theo:boundxtmode}
\end{align*}
which completing the proof.
\end{proof}

{\bf \noindent Proof to Theorem \ref{theo_1}}

\begin{proof}
%From Lemma \ref{lemma:boundfplus}, we shall see that we need to bound 

Setting $\gamma_t = \gamma$, then from Lemma~\ref{lemma:boundfplus}, we have
\begin{align*}
\mathbb{E}\|\nabla f(\overline{X_t})\|^2 + \left(1-L\gamma\right)\mathbb{E}\|\overline{\nabla f}(X_t)\|^2
\leq & \frac{2}{\gamma}\left(\mathbb{E}f(\overline{X}_t)-f^*-\left(\mathbb{E}f(\overline{X}_{(t+1)})-f^*\right)\right)\\
	&+ \frac{L^2}{n}\sum_{i=1}^{n}\|\overline{X}_t-x_t^{(i)}\|^2 + \frac{L\tilde{\sigma}^2}{nt\gamma}+  \frac{L\gamma}{n}\sigma^2. \numberthis\label{boundxplusfinal}
\end{align*}
%Actually, the key {\color{red}in the proof} is {\color{red}to ensure all nodes make the consensus}, {\color{red}that} is to bound $\sum_{i=1}^n\left\|\overline{X}_t-\bm{x}_t^{(i)}\right\|^2$. 
From Lemma~\ref{lemma:boundx}, we have
\begin{align*}
\sum_{t=1}^{T}\sum_{i=1}^n\left(1-6C_1L^2\gamma^2\right)\mathbb{E}\left\|\overline{X}_t-x_t^{(i)}\right\|^2
\leq &\frac{2n\tilde{\sigma}^2}{1-\rho^2}\log T + 2\left(\sigma^2+3\zeta^2\right)nC_1T\gamma^2\\
&+6C_1n\sum_{t=1}^{T-1}\mathbb{E}\gamma^2\left\|\nabla f(\overline{X}_t)\right\|^2.
\end{align*}
If $\gamma$ is not too large that satisfies $1-6C_1L^2\gamma^2 > 0$, we have
\begin{align*}
\sum_{t=1}^{T}\sum_{i=1}^n\mathbb{E}\left\|\overline{X}_t-x_t^{(i)}\right\|^2 \leq
& \frac{1}{1-6C_1L^2\gamma^2}\left( \frac{2n\tilde{\sigma}^2}{1-\rho^2}\log T + 2\left(\sigma^2+3\zeta^2\right)nC_1T\gamma^2\right)\\
&+\frac{6C_1n}{1-6C_1L^2\gamma^2}\sum_{t=1}^{T-1}\mathbb{E}\gamma^2\left\|\nabla f(\overline{X}_t)\right\|^2.\numberthis\label{theo:boundxtfinal}
\end{align*}
Summarizing both sides of \eqref{boundxplusfinal} and applying \eqref{theo:boundxtfinal} yields
\begin{align*}
& \sum_{t=1}^T\left(\mathbb{E}\|\nabla f(\overline{X_t})\|^2 + \left(1-L\gamma\right)\mathbb{E}\|\overline{\nabla f}(X_t)\|^2\right) \\
\leq & \frac{2\mathbb{E}f(\overline{X}_1)-2f^*}{\gamma} +\frac{L^2}{n}\sum_{t=1}^{T-1}\sum_{i=1}^n\mathbb{E}\left\|\overline{X}_t-x_t^{(i)}\right\|^2 + \frac{L\log T}{n\gamma}\tilde{\sigma}^2 + \frac{LT\gamma}{n}\sigma^2\\
\leq & \frac{2(f(0)-f^*)}{\gamma} +
\frac{L\log T}{n\gamma}\tilde{\sigma}^2 + \frac{LT\gamma}{n}\sigma^2 + \frac{4C_2\tilde{\sigma}^2L^2}{1-\rho^2}\log T\\
& +4L^2C_2\left(\sigma^2+3\zeta^2\right)C_1T\gamma^2 + 12L^2C_2C_1\sum_{t=1}^{T-1}\mathbb{E}\gamma^2\left\|\nabla f(\overline{X}_t)\right\|^2,
\end{align*}
where $C_2 = \frac{1}{1-6C_1L^2\gamma^2}$. It implies
\begin{align*}
& \sum_{t=1}^T\left(\left(1-C_3\right)\mathbb{E}\|\nabla f(\overline{X}_t)\|^2 + C_4\mathbb{E}\|\overline{\nabla f}(X_t)\|^2\right)
\\ \leq & \frac{2(f(0)-f^*)}{\gamma} +
\frac{L\log T}{n\gamma}\tilde{\sigma}^2 + \frac{LT\gamma}{n}\sigma^2 + \frac{4C_2\tilde{\sigma}^2L^2}{1-\rho^2}\log T + 4L^2C_2\left(\sigma^2+3\zeta^2\right)C_1T\gamma^2,
%\numberthis \label{theo:finalexpression}
\end{align*}
where $C_3 = 12L^2C_2C_1\gamma^2$ and $C_4 = \left(1-L\gamma\right)$. It completes the proof.
\end{proof}

{\bf \noindent Proof to \Cref{cor:convergence}}
\begin{proof}
when $\gamma=\frac{1}{12\sqrt{C_1}L+\frac{\sigma}{\sqrt{n}}T^{\frac{1}{2}} + \zeta^{\frac{2}{3}}T^{\frac{1}{3}} }$, we have
\begin{align*}
1 - L\gamma \geq & 0,\\
C_2 =  1 -6C_1L^2\gamma^2 < & 2,\\
12L^2C_2C_1\gamma^2 \leq &\frac{1}{2}.
\end{align*}
So we can remove the $\left(1-L\gamma\right)\mathbb{E}\|\overline{\nabla f}(X_t)\|^2$ term on the left side of \eqref{eq_theo_1} and substitute $12L^2C_2C_1\gamma^2$ with $\frac{1}{2}$, then we have

\begin{align*}
\sum_{t=1}^T \frac{1}{2}\mathbb{E}\|\nabla f(\overline{X_t})\|^2 
\leq & 2(f(0) - f^*)6\sqrt{2C_1}L + \frac{2(f(0)-f^*)\sigma)}{\sqrt{n}}T^{\frac{1}{2}} + 2(f(0) - f^*)\zeta^{\frac{2}{3}}T^{\frac{1}{3}}\\
& + \frac{6\sqrt{2C_1}L^2\tilde{\sigma}^2}{n}\log T  + \frac{L\log T\sigma\tilde{\sigma}^2}{n\sqrt{n}}T^{\frac{1}{2}} +  \frac{L\log{T}\tilde{\sigma}^2\zeta^{\frac{2}{3}}}{n}T^{\frac{1}{3}}\\
& + \frac{L\sigma T^{\frac{1}{2}}}{\sqrt{n}} + \frac{8\tilde{\sigma}^2L^2\log T}{1-\rho^2}\\
& + \frac{8nL^2\sigma^2C_1T}{72C_1nL^2+\sigma^2T} + \frac{24nL^2\zeta^2C_1T}{72C_1nL^2+ \sigma^2T + n\zeta^{\frac{4}{3}}T^{\frac{2}{3}}}.
\end{align*}
\begin{align*}
\frac{1}{T}\sum_{t=1}^T\mathbb{E}\|\nabla f(\overline{X_t})\|^2
\leq & \sigma\frac{4(f(0)-f^*) + 4L}{\sqrt{nT}} + \zeta^{\frac{2}{3}}\frac{4(f(0)-f^*) + 24L^2C_1}{T^{\frac{2}{3}}} + \tilde{\sigma}^2\frac{10L^2\log{T}}{(1-\rho^2)T}\\
& + \sigma\tilde{\sigma}^2\frac{2L\log{T}}{n\sqrt{nT}} + \zeta^{\frac{2}{3}}\tilde{\sigma}^2\frac{L\log{T}}{nT^{\frac{2}{3}}} + \frac{2(f(0)-f^*)L}{T} + \sigma^2\frac{8nL^2C_1}{72nC_1L^2 + \sigma^2T} ,
\end{align*}
which means 
\begin{align*}
\frac{1}{T}\sum_{t=1}^T\mathbb{E}\|\nabla f(\overline{X_t})\|^2
\lesssim & \frac{\sigma}{\sqrt{nT}} + \frac{\zeta^{\frac{2}{3}}}{T^{\frac{
2}{3}}} + \frac{1}{T} + \frac{\tilde{\sigma}^2\sigma\log T}{n\sqrt{nT}}+\frac{\zeta^{\frac{2}{3}}\tilde{\sigma}^2\log{T}}{nT^{\frac{2}{3}}} + \frac{\tilde{\sigma}^2\log T}{T}.\numberthis \label{alg1_coro_1}
\end{align*}
From Lemma~\ref{lemma:boundx}, we have
\begin{align*}
\frac{1}{T}\sum_{t=1}^{T}\sum_{i=1}^n\left(1-6C_1L^2\gamma_t^2\right)\mathbb{E}\left\|\overline{X}_t-\bm{x}_t^{(i)}\right\|^2
\leq &\frac{2n\tilde{\sigma}^2\log T}{(1-\rho^2)T} + 2\left(\sigma^2+3\zeta^2\right)nC_1\gamma^2\\
&+6C_1n\gamma^2\frac{1}{T}\sum_{t=1}^{T-1}\mathbb{E}\left\|\nabla f(\overline{X}_t)\right\|^2.
\end{align*}
Combing it with \eqref{alg1_coro_1}, we have
\begin{align*}
\frac{1}{T}\sum_{t=1}^{T}\sum_{i=1}^n\mathbb{E}\left\|\overline{X}_t-\bm{x}_t^{(i)}\right\|^2
\lesssim & \frac{n\tilde{\sigma}^2\log T}{T} + \frac{n\sqrt{n}}{T} + \frac{n\zeta^{\frac{2}{3}}}{T^{\frac{2}{3}}} + \frac{1}{T^2}+ \frac{\tilde{\sigma}^2\sigma\log T}{T^{\frac{3}{2}}}+\frac{\zeta^{\frac{2}{3}}\tilde{\sigma}^2\log{T}}{T^{\frac{5}{3}}}.\numberthis\label{final_coro2}
\end{align*}

\section{Why naive combination between compression and D-PSGD does not work?}\label{sec:supp:naive}

Consider combine compression with the D-PSGD algorithm. Let the compression of exchanged models $X_t$ be
\[
\tilde{X}_t = C(X_t) = X_t + Q_t,
\]
where $Q_t=[\bm{q}_t^{(1)},\bm{q}_t^{(2)},\cdots,\bm{q}_t^{(n)}]$, and $\bm{q}_t^{(i)}=\tilde{\bm{x}}_t^{(i)}-\bm{x}^{(i)}_t$ is the random noise. Then the update iteration becomes
\begin{align*}
X_{t+1} = &\tilde{X}_t W - \gamma_t G(X_t; \xi_t)\\
 = & X_t W + \underbrace{Q_tW}_{\text{not diminish}}- \gamma_t G(X_t; \xi_t).
\end{align*}
This naive combination does not work, because the compression error $Q_t$ does not diminish unlike the stochastic gradient variance that can be controlled by $\gamma_t$ either decays to zero or is chosen to be small enough. 
\end{proof}

\end{document}